\documentclass[accepted]{uai2021} % for initial submission
\usepackage[american]{babel}
\usepackage{natbib} % has a nice set of citation styles and commands
\usepackage{mathtools} % amsmath with fixes and additions
\usepackage{booktabs} % commands to create good-looking tables
\usepackage{tikz} % nice language for creating drawings and diagrams

\usepackage{hyperref}
\usepackage{url}

\usepackage{amsmath,amsfonts,amsthm,amssymb}

\usepackage[utf8]{inputenc} % allow utf-8 input
\usepackage[T1]{fontenc}    % use 8-bit T1 fonts
\usepackage{nicefrac}       % compact symbols for 1/2, etc.
\usepackage{microtype}      % microtypography

\usepackage{times}
\usepackage{color}
\usepackage{comment}
\usepackage{caption}
\usepackage{float}
\usepackage{graphicx}
\usepackage{subfigure}
\usepackage{multirow}
\usepackage{mathrsfs}
\usepackage{newfloat}
\usepackage{relsize}
\usepackage{dsfont}

\usepackage{enumitem}
\usepackage{nicefrac}
\usepackage{caption}
\usepackage{siunitx,caption}

\usepackage{wrapfig}
\usepackage{graphicx}
\usepackage{subfigure}

\usepackage{algorithm}
\usepackage{algorithmic}

\newcommand{\B}{\mathcal{B}}

\newcommand{\E}{\mathbb{E}}
\newcommand{\D}{\mathcal{D}}

\newcommand{\R}{\mathbb{R}}

\renewcommand{\P}{\mathcal{P}}

\newcommand{\ind}{\mathbb{I}}

\def\pdf{{\text{pdf}}}
\def\query{{\text{query\_until\_lower}}}
\def\cont{{\text{continue\_at\_min}}}

\def\pdfn{{\text{pdf}_n}}
\def\pdfe{{\text{pdf}_e}}
\def\ls{{\texttt{LS}}}

\let\cite\citep

\numberwithin{equation}{section}
\numberwithin{figure}{section}

\theoremstyle{plain}
	\newtheorem{theorem}{Theorem}[section]

	\newtheorem{corollary}[theorem]{Corollary}

\theoremstyle{definition}

	\newtheorem*{remark*}{Remark}

\usepackage{thm-restate}
\title{Exploring the Loss Landscape in Neural Architecture Search}

\author[1]{Colin White}
\author[1]{Sam Nolen}
\author[1,2]{Yash Savani}
\affil[1]{%
    Abacus.AI. \texttt{\{colin, sam\}@abacus.ai}
}
\affil[2]{%
    Carnegie Mellon University. \texttt{ysavani@cs.cmu.edu}
}
\begin{document}
\maketitle

\begin{abstract}

Neural architecture search (NAS) has seen a steep rise in interest over the last
few years.
Many algorithms for NAS consist of searching through a space of architectures
by iteratively choosing an architecture, evaluating its performance by training it,
and using all prior evaluations to come up with the next choice.
The evaluation step is noisy -
the final accuracy 
varies based on the random initialization of the weights.
Prior work has focused on devising new search algorithms to handle this noise, 
rather than quantifying or understanding the level of 
noise in architecture evaluations.
In this work, we show that \emph{(1)} the simplest hill-climbing algorithm
is a powerful baseline for NAS, and 
\emph{(2)}, when the noise in popular NAS benchmark datasets 
is reduced to a minimum,
hill-climbing to outperforms 
many popular state-of-the-art algorithms.
We further back up this observation by showing that the number of local minima 
is substantially reduced as the noise decreases,
and by giving a theoretical characterization of the performance of local search in NAS.
Based on our findings, for NAS research we suggest
\emph{(1)} using local search as a baseline, and \emph{(2)}
denoising the training pipeline when possible.
\end{abstract}

\begin{comment}
Neural architecture search (NAS) has seen a steep rise in interest over the last few years. Many algorithms for NAS consist of searching through a space of architectures by iteratively choosing an architecture, evaluating its performance by training it, and using all prior evaluations to come up with the next choice. The evaluation step is noisy - the final accuracy varies based on the random initialization of the weights. Prior work has focused on devising new search algorithms to handle this noise, rather than quantifying or understanding the level of noise in architecture evaluations. In this work, we show that (1) the simplest hill-climbing algorithm is a powerful baseline for NAS, and (2), when the noise in popular NAS benchmark datasets is reduced to a minimum, hill-climbing to outperforms many popular state-of-the-art algorithms. We further back up this observation by showing that the number of local minima is substantially reduced as the noise decreases, and by giving a theoretical characterization of the performance of local search in NAS. Based on our findings, for NAS research we suggest (1) using local search as a baseline, and (2) denoising the training pipeline when possible.

\end{comment}

\section{Introduction} \label{sec:intro}

Neural architecture search (NAS) is a widely popular area of machine learning 
which seeks to automate the development of the best neural network for a given dataset. 
Many methods for NAS have been proposed, including reinforcement learning, 
gradient descent, and Bayesian optimization~\citep{nas-survey,zoph2017neural, darts}.
Many popular NAS algorithms can be instantiated by the optimization problem
$\min_{a\in A} f(a)$, where $A$ denotes a set of architectures (the \emph{search space})
and $f(a)$ denotes the objective function for $a$, often set to the validation
accuracy of $a$ after training using a fixed set of hyperparameters.
With the release of three NAS benchmark datasets~\citep{nasbench, nasbench201, nasbench301},
the extreme computational cost for NAS is no longer a barrier, 
and it is easier to fairly compare different algorithms.
However, recent work has noticed that the training step used in these benchmarks
is quite stochastic~\cite{nasbench, nasbench301}.
%due to the random weight initialization.
In NAS, where the goal is to search over complex neural architectures, 
a noisy reward function makes the problem even more challenging.
In fact, recent innovations are designed specifically to handle the noisy objective 
function~\citep{real2019regularized, zaidi2020neural}.
A natural question is therefore, how much of the complexity of NAS can be
attributed to the noise in the training pipeline?

In this work, 
\footnote{
%See the full-length paper here: \url{https://arxiv.org/abs/2005.02960}.
Our code is available at \url{https://github.com/naszilla/naszilla}.}
we answer this question by showing that the difficulty of NAS is highly
correlated with the noise in the training pipeline. 
We show that \emph{(1)} the simplest local search algorithm (hill-climbing) 
is already a strong baseline in NAS, and
\emph{(2)}, when the noise in the training pipeline is reduced to a minimum,
local search is sufficient to outperform many state-of-the-art techniques.

Local search is a simple and canonical greedy algorithm in combinatorial optimization 
and has led to famous results in the study of approximation 
algorithms~\citep{michiels2007theoretical, cohen2016local, friggstad2019local}.
However, local search has been neglected in the field of NAS; a recent paper
even suggests that it performs poorly due to the number of local minima throughout 
the search space~\citep{alphax}.
The most basic form of local search, often called the hill-climbing algorithm,
consists of starting with a random architecture
and then iteratively training all architectures in its neighborhood, 
choosing the best one for the next iteration.
The neighborhood is typically defined as all architectures which differ by one 
operation or edge.
Local search finishes when it reaches a (local or global) optimum,
or when it exhausts its runtime budget.

We show that on NAS-Bench-101, 201, 
and 301/ DARTS~\citep{nasbench, nasbench201, nasbench301, darts}, 
if the noise is reduced to a minimum,
then local search is competitive with all popular state-of-the-art
NAS algorithms.
This result is especially surprising because local search, which can be implemented in
five lines of code (Algorithm~\ref{alg:local_search}), is in stark contrast to most
state-of-the-art NAS algorithms, which have many moving parts and 
even use neural networks as subroutines~\cite{wen2019neural,shi2019multi}.
This suggests that the complexity in prior methods may have been developed to deal with
the noisy reward function.
We also experimentally show that as the noise in the training pipeline
increases, the number of local minima increases, and the basin of
attraction of the global minimum decreases. These results further
suggest that NAS becomes easier when the noise is reduced.

Motivated by these findings, we also present a theoretical study
to better understand the performance of local search under different levels of noise.
The underlying optimization problem in NAS is a hybrid between discrete
optimization, on a graph topology, and continuous optimization, on the
distribution of architecture accuracies.
We formally define a NAS problem instance by the graph topology, 
a global probability density function (PDF) on the architecture accuracies, 
and a local PDF on the accuracies between neighboring architectures, 
and we derive a set of equations which
calculate the probability that a randomly drawn architecture will converge to within $\epsilon$
of the global optimum, for all $\epsilon>0$.
As a corollary, we give equations for the expected number of local minima,
and the expected size of the preimage of a local minimum.
These results completely characterize the performance of local search.
To the best of our knowledge, this is the first result which theoretically predicts the
performance of a NAS algorithm, and may be of independent interest within discrete optimization.
We run simulations which suggest that our theoretical results predict the performance
on real datasets reasonably well.

Our findings raise a few points for the field of NAS.
Since much of the difficulty is in the stochasticity of the training pipeline, 
denoising the training pipeline as much as possible is worthwhile for
future work.
Second, our work suggests that 
local methods for NAS may be promising. 
That is, methods which explore the 
search space by iteratively making small edits to the best architectures found so far.
Furthermore, we suggest using local search as a baseline for future work.
We release our code, and we discuss our adherence to the NAS research checklist~\citep{lindauer2019best}. % in Appendix~\ref{app:experiments}.

\paragraph{Our contributions.} We summarize our contributions.
\begin{itemize} [topsep=0pt, itemsep=2pt, parsep=0pt, leftmargin=5mm]
    \item 
    We show that local search is a strong baseline in its own right,
    and outperforms many state-of-the-art NAS
    algorithms across three popular benchmark datasets when the 
    noise in the training pipeline is reduced to a minimum.
    We also show that the number of local minima increases as the noise in the training pipeline increases.
    Our results suggest that making the training pipeline in NAS more
    consistent, is just as worthwhile as coming up with novel
    search algorithms.
    \item 
    We give a theoretical characterization of the properties of a dataset 
    necessary for local search to give strong performance. 
    We experimentally validate these results on real neural architecture
    search datasets.
    Our results improve the theoretical understanding of local search and
    lay the groundwork for future studies.
    \end{itemize}

\section{Related Work} \label{sec:related}

%\paragraph{Local search in theoretical computer science.}
Local search has been studied since at least the 1950s in the context of the 
traveling salesman problem~\citep{bock1958algorithm, croes1958method},
machine scheduling~\citep{page1961approach},
and graph partitioning~\citep{kernighan1970efficient}.
Local search has consistently seen significant attention in 
theory~\citep{aarts1997local, balcan2020k, johnson1988easy}
and practice~\citep{bentley1992fast, johnson1997traveling}.
There is also a large variety of work in local optimization of noisy
functions, handling the noise by averaging the objective function over multiple
evaluations~\citep{rakshit2017noisy, akimoto2015analysis}, 
using surrogate models~\citep{booker1998optimization, caballero2008algorithm}, 
or using regularization~\citep{real2019regularized}.

%\paragraph{Neural architecture search.}
NAS has gained popularity in
recent years~\citep{kitano1990designing, stanley2002evolving, zoph2017neural},
%recent years~\citep{zoph2017neural},
although the first few techniques have been around since at least the 1990s
~\citep{yao1999evolving, amoebanet}.
Common techniques include
Bayesian optimization~\citep{nasbot, auto-keras}, 
reinforcement learning \citep{zoph2017neural, enas, pnas}, 
gradient descent \citep{darts, gdas},
prediction-based~\citep{bananas, shi2019multi, white2021powerful},
evolution~\citep{maziarz2018evolutionary, real2019regularized},
and using novel encodings to improve the 
search~\citep{white2020study, yan2020does, yan2020cate}.

Recent papers have highlighted the need for fair and reproducible NAS 
comparisons~\citep{randomnas, lindauer2019best},
spurring the release of three
NAS benchmark datasets~\citep{nasbench, nasbench201, nasbench301},
each of which utilize tens of thousands of pretrained neural networks.
See the recent survey~\citep{nas-survey} for a more comprehensive overview on NAS.

There has been some prior work using local search for NAS.
\citet{elsken2017simple} use local search with network morphisms
guided by cosine annealing, which is a more complex variant.
\citet{alphax} use local search as a baseline, but kill the run
after encountering a local minimum rather than using the remaining
runtime budget to start a new run.
Concurrent work has also shown that simple local search 
is a strong baseline on NASBench-101~\citep{ottelander2020local}
for multi-objective NAS 
(where the objective is a function of accuracy and network complexity).
This work focuses on macro search rather than cell-based search,
and does not investigate the effect of noise on the performance.
The existence of this work strengthens one of our conclusions
(that local search is a strong NAS algorithm) because it is now 
independently verified.

\section{Preliminaries} \label{sec:prelim}
In this section, we formally define the local search algorithm and 
notation that will be used for the rest of the paper.
Given a set $A$, denote an objective function $\ell: A\rightarrow [0,\infty)$.
We refer to $A$ as a search space of neural architectures, 
and $\ell(v)$ as the expected validation loss of $v\in A$ over a fixed dataset and training pipeline.
When running a NAS algorithm, we have access to a noisy version of
$\ell$, i.e., when we train an architecture $a$, we receive 
loss$=\ell(v)+x$ for noise $x$ drawn from a distribution $\D_v$
(we explore different families of distributions in 
Sections~\ref{sec:experiments} and~\ref{sec:method}).
The goal is to find $v^*=\text{argmin}_{v\in A} \ell(v)$, the neural architecture with the 
minimum validation loss, or an architecture whose validation loss is within 
$\epsilon$ of the minimum, for some small $\epsilon>0$.
We define a neighborhood function $N:A\rightarrow 2^A$.
For instance, $N(v)$ might represent 
the set of all neural architectures which differ from $v$ by one operation or edge.

Local search in its simplest form (also called the hill-climbing algorithm) 
is defined as follows. 
Start with a random architecture $v$ and evaluate $\ell(v)$ by training $v$.
Iteratively train all architectures in $N(v)$, and then replace $v$ with the architecture $u$ 
such that $u=\text{argmin}_{w\in N(v)}\ell(w).$
Continue until we reach an architecture $v$
such that $\forall u\in N(v),~\ell(v)\leq \ell(u)$, i.e., we reach a local minimum.
See Algorithm~\ref{alg:local_search}.
We often place a runtime bound on the algorithm, in which case the algorithm returns the
architecture $v$ with the lowest value of $\ell(v)$ when it exhausts the runtime budget.
In Section~\ref{sec:experiments}, 
we also consider two simple variants.
In the $\query$ variant, instead of evaluating every architecture 
in the neighborhood $N(v)$ and picking the best one, 
we draw architectures $u\in N(v)$ at random without replacement
and move to the next iteration as soon as $\ell(u)<\ell(v).$
In the $\cont$ variant, we do not stop at a local minimum, instead
moving to the second-best architecture found so far and continuing until we exhaust
the runtime budget.
One final variant, which we explore in 
%the supplementary material,
Appendix~\ref{app:experiments},
is choosing $k$ initial architectures at random,
and setting $v_1$ to be the architecture with the lowest objective value.

\begin{algorithm}
\caption{Local search}\label{alg:local_search}
\begin{algorithmic} %[1]
\STATE {\bfseries Input:} Search space $A$, objective function $\ell$, \\
\quad\quad\quad neighborhood function $N$
\STATE 1. Pick an architecture $v_1\in A$ uniformly at random
\STATE 2. Evaluate $\ell(v_1)$; denote a dummy variable $\ell(v_0)=\infty$; \\
\quad set $i=1$
\STATE 3. While $\ell(v_i)<\ell(v_{i-1}):$
\begin{enumerate}[label=\roman*., itemsep=1pt, parsep=1mm, topsep=1pt, leftmargin=8mm]
    \item Evaluate $\ell(u)$ for all $u\in N(v_i)$
    \item Set $v_{i+1}=\text{argmin}_{u\in N(v_i)} \ell(u)$; set $i=i+1$
\end{enumerate}
\STATE {\bfseries Output:} Architecture $v_i$
\end{algorithmic}
\end{algorithm}

%\subsection{Notation}
\paragraph{Notation.}
Now we define the notation used in Sections~\ref{sec:experiments} and~\ref{sec:method}.
Given a search space $A$ and a neighborhood function $N$,
we define the neighborhood graph $G_N=(A, E_N)$ such that for $u,v\in A$, 
the edge $(u,v)$ is in $E_N$ if and only if $v\in N(u)$.
We assume that $v\in N(u)$ implies $u\in N(v)$, therefore, the neighborhood graph is undirected.
We only consider symmetric neighborhood functions, that is, $v\in N(u)$ implies $u\in N(v)$. 
Therefore, we may assume that the neighborhood graph is undirected.
%The \emph{diameter} of a graph, $D(G)$, is the maximum shortest path between any two points.
%, formally, $D(G)=\max_{u,v\in A}d(u,v)$, where $d(u,v)$ denotes the minimum number of edges between $u$ and $v$ in $G$.
%In Section~\ref{sec:experiments}, we will see that for NASBench-201, $G_N=(K_5)^6$ 
%(the Cartesian product of 6 cliques of size 5) and therefore $D(G_N)=6$.
Given $G$, $N$, and a loss function $\ell$, define $\ls:A\rightarrow A$ such
that $\forall v\in A$, $\ls(v)=\text{argmin}_{u\in N(v)} \ell(u)$ if 
$\min_{u\in N(v)}\ell(u)<\ell(v)$, and $\ls(v)=\emptyset$ otherwise.
In other words, $\ls(v)$ denotes the architecture after performing one iteration 
of local search starting from $v$. 
See Figure~\ref{fig:local_search} for an example.
For integers $k\geq 1$, recursively define $\ls^k(v)=\ls(\ls^{k-1}(v))$.
We set $\ls^0(v)=v$ and denote $\ls^*(v)=\min_{k\mid \ls^k(v)\neq\emptyset}\ls^k(v)$,
that is, the output when running local search to convergence, starting at $v$.
Similarly, define the preimage $\ls^{-k}(v)=\{u\mid \ls^{k}(u)=v\}$ for integers $k\geq 1$
and $\ls^{-*}(v)=\{u\mid \exists k\geq 0\text{ s.t. }\ls^{-k}(u)=v\}$.
That is, $\ls^{-*}(v)$ is a multifunction which defines the set of all points $u$ 
which reach $v$ at some point during local search.
We refer to $\ls^{-*}$ as the full preimage of $v$.

\section{Experiments} \label{sec:experiments}

In this section, we discuss our experimental setup and results.
To promote reproduciblity,
we discuss how our experiments follow the best practices 
checklist~\citep{lindauer2019best} 
%in the supplementary material,
in Appendix~\ref{app:experiments},
and we release our code
at \url{https://github.com/naszilla/naszilla}.
We start by describing the benchmark datasets used in our experiments.

\paragraph{NAS benchmarks.}
To conduct our experiments, 
we use three of the most popular NAS benchmarks: NASBench-101, 201, and 301/DARTS.
NASBench-101 consists of over 423,000 
unique neural architectures with precomputed validation 
and test accuracies for 108 epochs on CIFAR-10. 
The cell-based search space consists of five nodes which can take on any DAG
structure, and each node can be one of three operations.
Each architecture was trained a total of three times using different random seeds.
The NASBench-201 dataset consists of $5^6=15,625$ unique neural architectures, 
with precomputed validation and test accuracies for 200 epochs on 
CIFAR-10, CIFAR-100, and ImageNet-16-120.
The search space consists of a cell which is a complete directed acyclic graph 
over 4 nodes. Therefore, there are $\binom{4}{2}=6$ edges.
Each \emph{edge} can be one of five operations.
As in NASBench-101, on each dataset, each architecture was trained three times
using different random seeds.

The DARTS~\citep{darts} search space is a popular search space for 
large-scale cell-based NAS experiments on CIFAR-10.
The search space contains roughly $10^{18}$ architectures, consisting 
of two cells: a convolutional cell and a reduction cell, each with six nodes. 
The first two nodes are input from previous layers, and the last four nodes 
can take on any DAG structure such that each node has degree two.
Each edge can take one of eight operations.
Recently, a surrogate benchmark, dubbed NASBench-301~\cite{nasbench301}, 
has been created for the DARTS search space.
The surrogate benchmark is created using an 
ensemble of XGBoost models~\cite{xgboost}, each initialized with different
random weights.

%\subsection{Local search performance}
\paragraph{Local search performance.}
We evaluate the relative performance of local search for NAS in settings 
with and without noise.
In a real-world NAS experiment, noise reduction can be achieved by modifying 
the training hyperparameters, which we discuss later in this section.
However, modifying the hyperparameters of the NASBench architectures is 
impractical due to the extreme computational cost needed to create the 
benchmarks~\citep{nasbench, nasbench201}. Instead, we artificially remove
much of the noise in these benchmarks using two different techniques.
On NASBench-101 and 201, 
where each architecture was independently trained three times,
the standard way to use the benchmark is to draw validation accuracies
at random. However, for each architecture, we can average all three
validation accuracies to obtain a less noisy estimate.
On NASBench-301, where the standard way to evaluate architectures is by 
using one estimate from the ensemble uniformly at random, we can
take the mean of all of the ensemble estimates.
This is shown to be less noisy even than the data used to train the 
ensemble itself~\citep{nasbench301}.
In general, we can easily control the noise of NASBench-301 by returning
the mean of the ensemble estimates plus a random normal variable with mean
0 and standard deviation equal to the standard deviation of the ensemble
estimates, multiplied by a constant $x$. $x=0$ corresponds to the denoised
setting, while $x=1$ corresponds to the standard setting.

We compare local search to random search, 
%DNGO~\citep{snoek2015scalable}, 
regularized evolution~\citep{real2019regularized},
Bayesian optimization, and BANANAS~\citep{bananas}.
%and NASBOT~\citep{nasbot}.
%On NASBench-101, we test local search
%with the aforementioned algorithms, as well as
%REINFORCE~\citep{reinforce} and AlphaX~\citep{alphax}.
For every algorithm, we used the code directly from the 
corresponding open source repositories.
For more details on the implementations, 
%see the supplementary material.
see Appendix~\ref{app:experiments}.
We gave each algorithm a budget of 300 evaluations.
For each algorithm, we recorded the test loss of the architecture with the best validation 
loss that has been queried so far.
We ran 200 trials of each algorithm and averaged the results.
For local search, we set $N(v)$ to denote all architectures which differ by one operation or edge.
If local search converged before its budget, it started a new run.
On NASBench-101 and 301, we used the $\query$ variant of local search, 
and on NASBench-201, we used the $\cont$ variant.
See Figure~\ref{fig:ls_baselines_201}.
On both NASBench-101 and 301, local search outperforms all other algorithms
when the noise is minimal,
amd performs similarly to Bayesian optimization in the standard
setting.
We include the results for all three datasets of NASBench-201 
in Appendix~\ref{app:experiments}.
%in the supplementary material.

\paragraph{Local minima statistics.}
Now we further show that denoised NAS is a simpler optimization problem
by computing statistics about the loss landscape of the noisy and
denoised search spaces.
We start by running experiments on NASBench-201.
Since this benchmark is only size 15625, the global minimum is known,
which allows us to compute the percent of architectures that converge
to the global minimum when running local search.
We also compute the number of local minima and average number of iterations
of local search before convergence.
We run experiments using the standard and denoised versions of NASBench-201
(defined earlier in this section), and we also use a fully randomized
version by replacing the validation error for each architecture 
with a number drawn uniformly from  $[0,1].$
For each experiment, we started local search from \emph{all} 15625 initial seeds
for local search, and averaged the results.
See Table~\ref{tab:nasbench_201}.
On the denoised search space, almost half of the 15625 architectures
converge to the global minimum, but under 7\% reach the global minimum on the
standard search space.
In Figure~\ref{fig:local_search}, we give a visualization of the
topologies of the denoised and fully random search spaces.

\begin{table}
\centering
\caption{Avg.\ num.\ of iterations until convergence, num.\ of local minima, and percent 
of initial architectures to reach the global minimum, for CIFAR-10 on NASBench-201.}
\begin{minipage}[c]{.48\textwidth}
\setlength\tabcolsep{0pt}
\begin{tabular*}{\textwidth}{l @{\extracolsep{\fill}}*{8}{S[table-format=1.4]}} 
\toprule
\multicolumn{1}{c}{Version}
& \multicolumn{1}{c}{\# iters} 
& \multicolumn{1}{c}{\# local min.} & 
\multicolumn{1}{c}{\% reached global min.} \\
\midrule
Denoised & 5.36 & \hspace{1cm}21 & 47.4 \\
Standard & 4.97 & \hspace{1cm}55 & 6.71 \\
Random & 2.56 & \hspace{.8cm}616 & 0.717 \\
\bottomrule
\end{tabular*} 
\label{tab:nasbench_201}
\end{minipage}
\end{table}

Finally, we run experiments on NASBench-301. 
Due to the extreme size ($10^{18}$), the global minimum is
not known. However, as described above, the surrogate nature of NASBench-301
allows for a more fine-grained control of the noise.
In Figure~\ref{fig:301_trends}, we plot the performance of NAS algorithms
with respect to the level of noise in the search space.
We also show that the average number of local search iterations needed for
convergence decreases with noise.

\begin{figure*}
\centering % 
\includegraphics[width=0.8\textwidth]{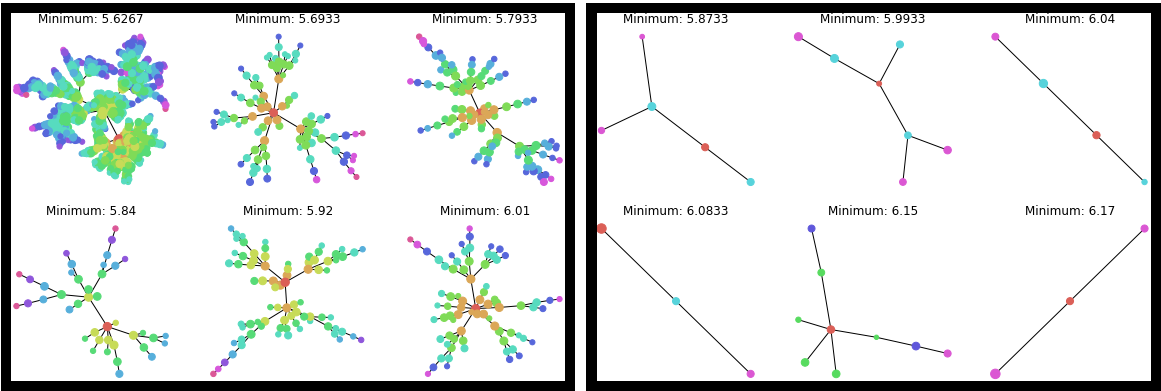}
\caption{
The local search tree for the architectures with the six lowest test losses
(colored red) on CIFAR-10 on NASBench-201, denoised (left) or fully random (right).
Each edge represents an iteration of local search, from the colder-colored node to the warmer-colored
node. While 47.4\% of architectures reach the global minimum in the denoised version, only
0.71\% of architectures reach the global minimum in the random version.
}
\label{fig:local_search}
\end{figure*}

\begin{figure*}
\centering % 
\includegraphics[width=0.33\textwidth]{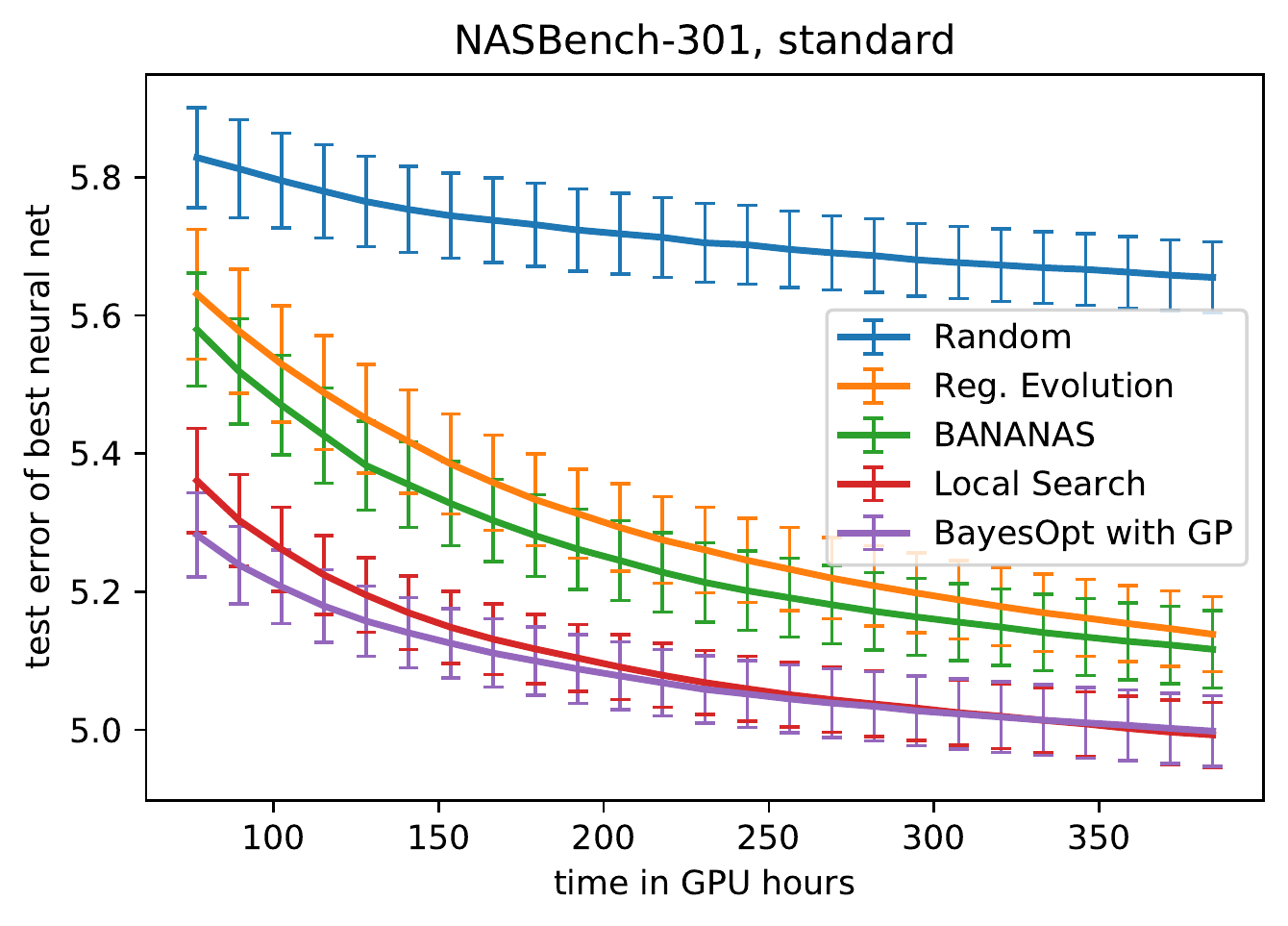}
\hspace{-3pt}
\includegraphics[width=0.33\textwidth]{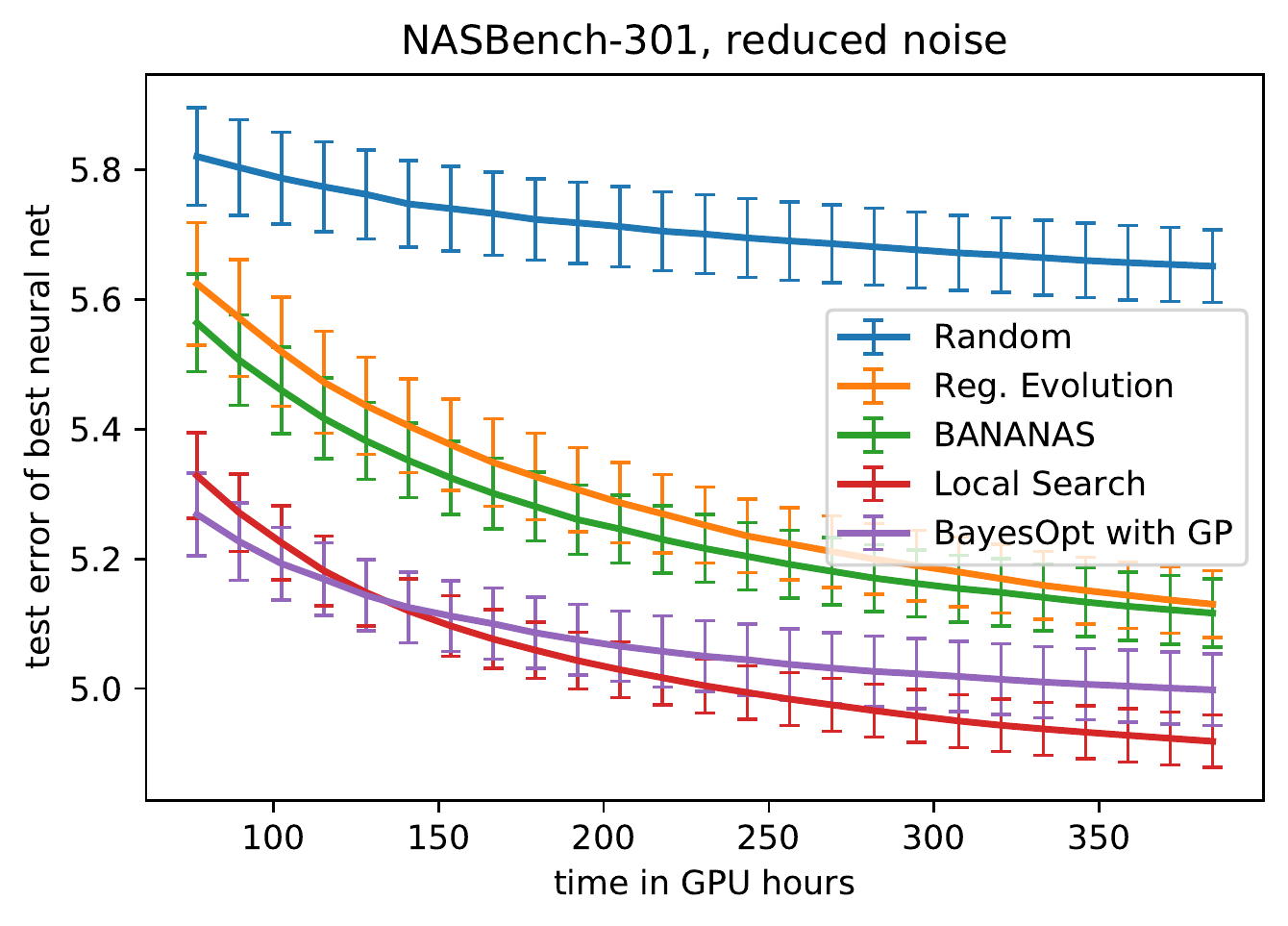}
\hspace{-3pt}
\includegraphics[width=0.33\textwidth]{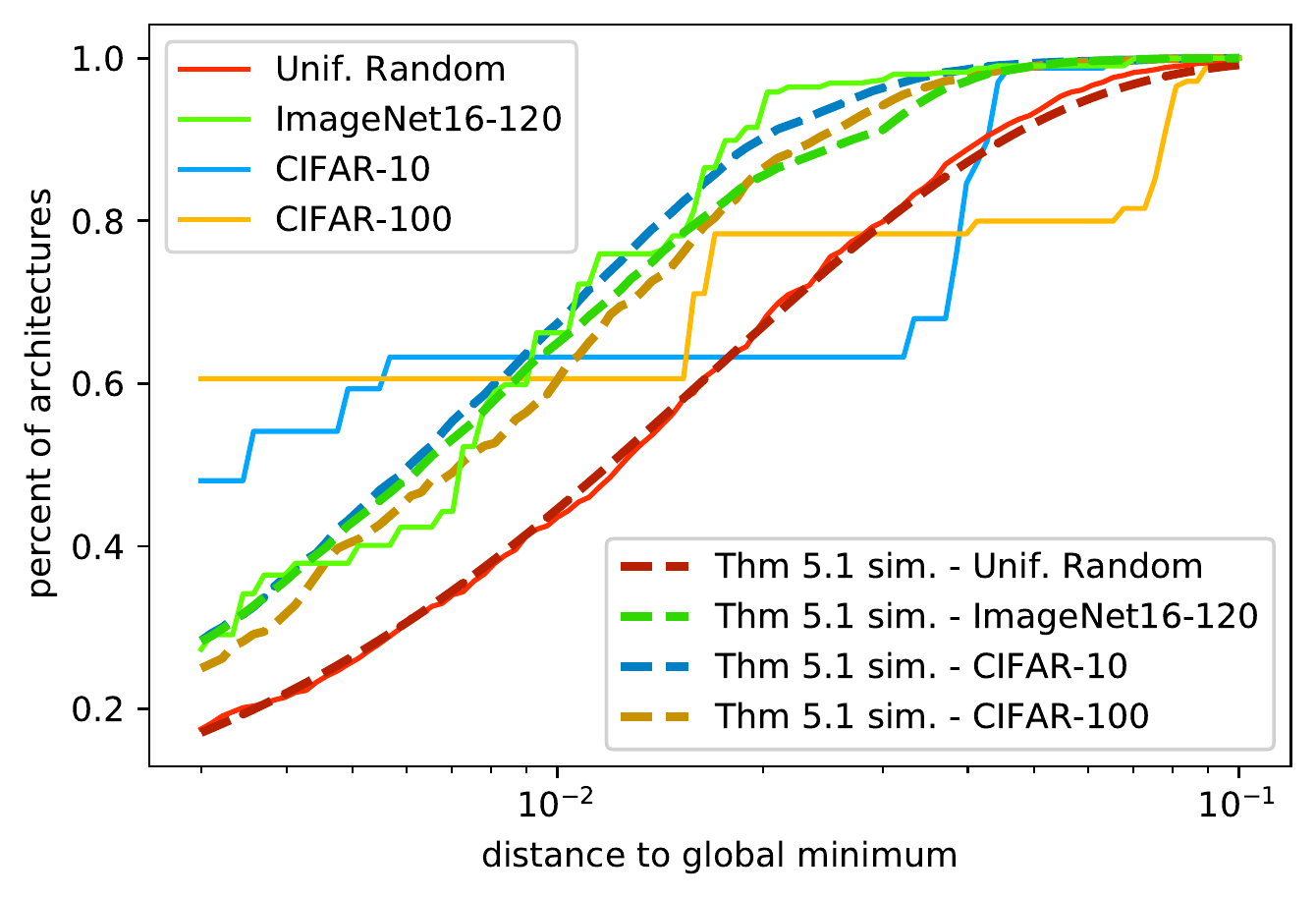}
\includegraphics[width=0.33\textwidth]{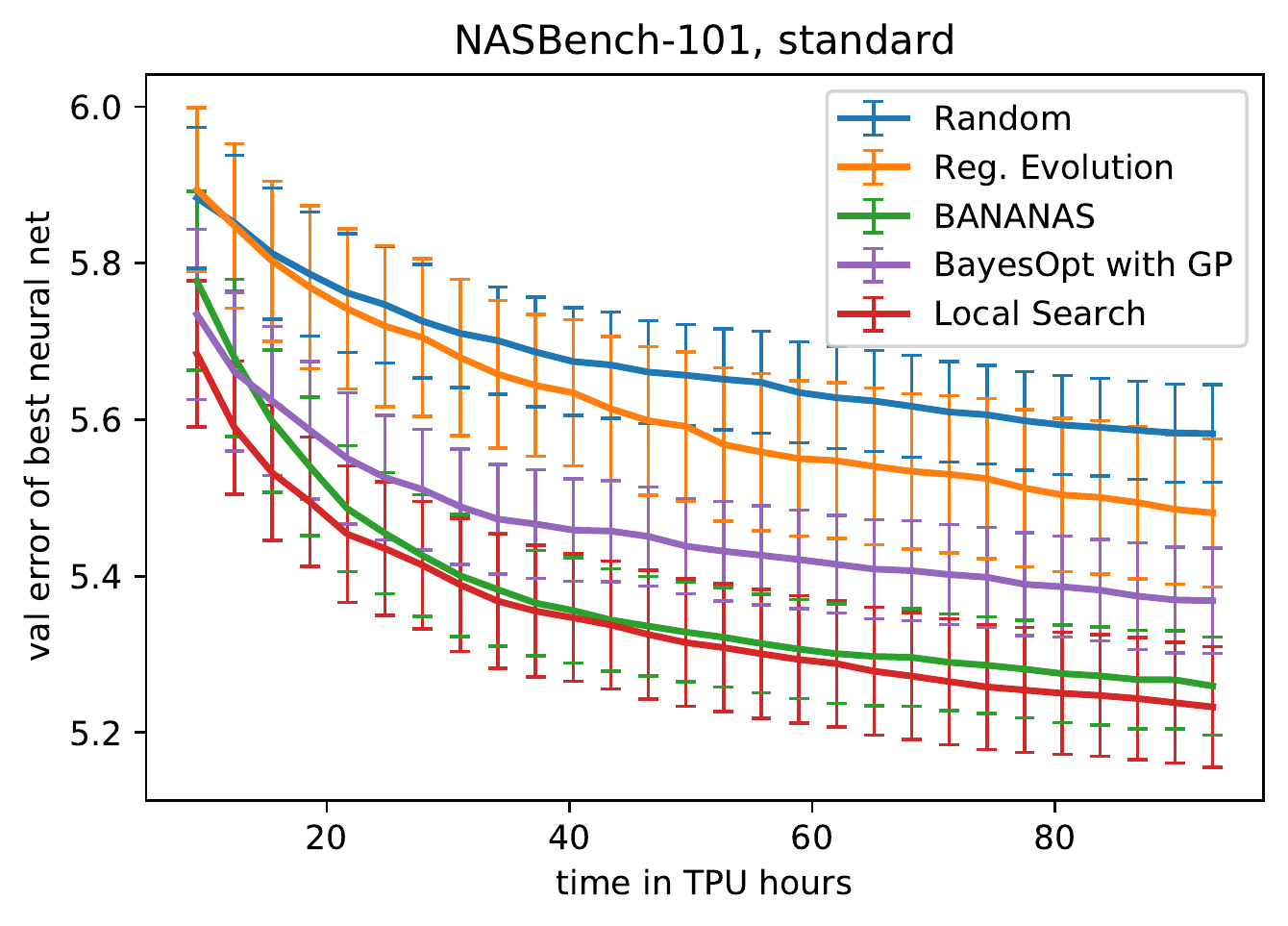}
\hspace{-3pt}
\includegraphics[width=0.33\textwidth]{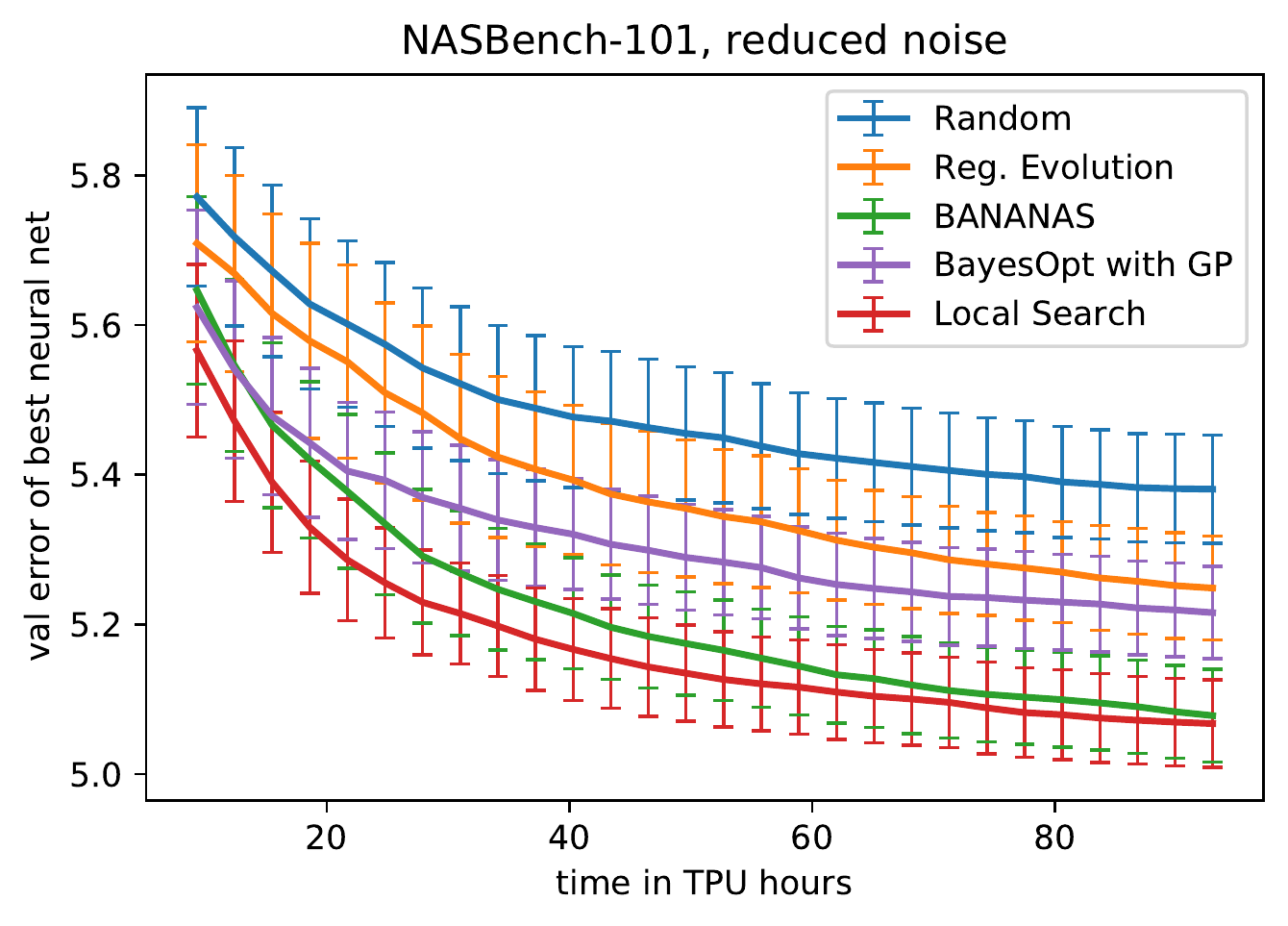}
\hspace{-3pt}
\includegraphics[width=0.33\textwidth]{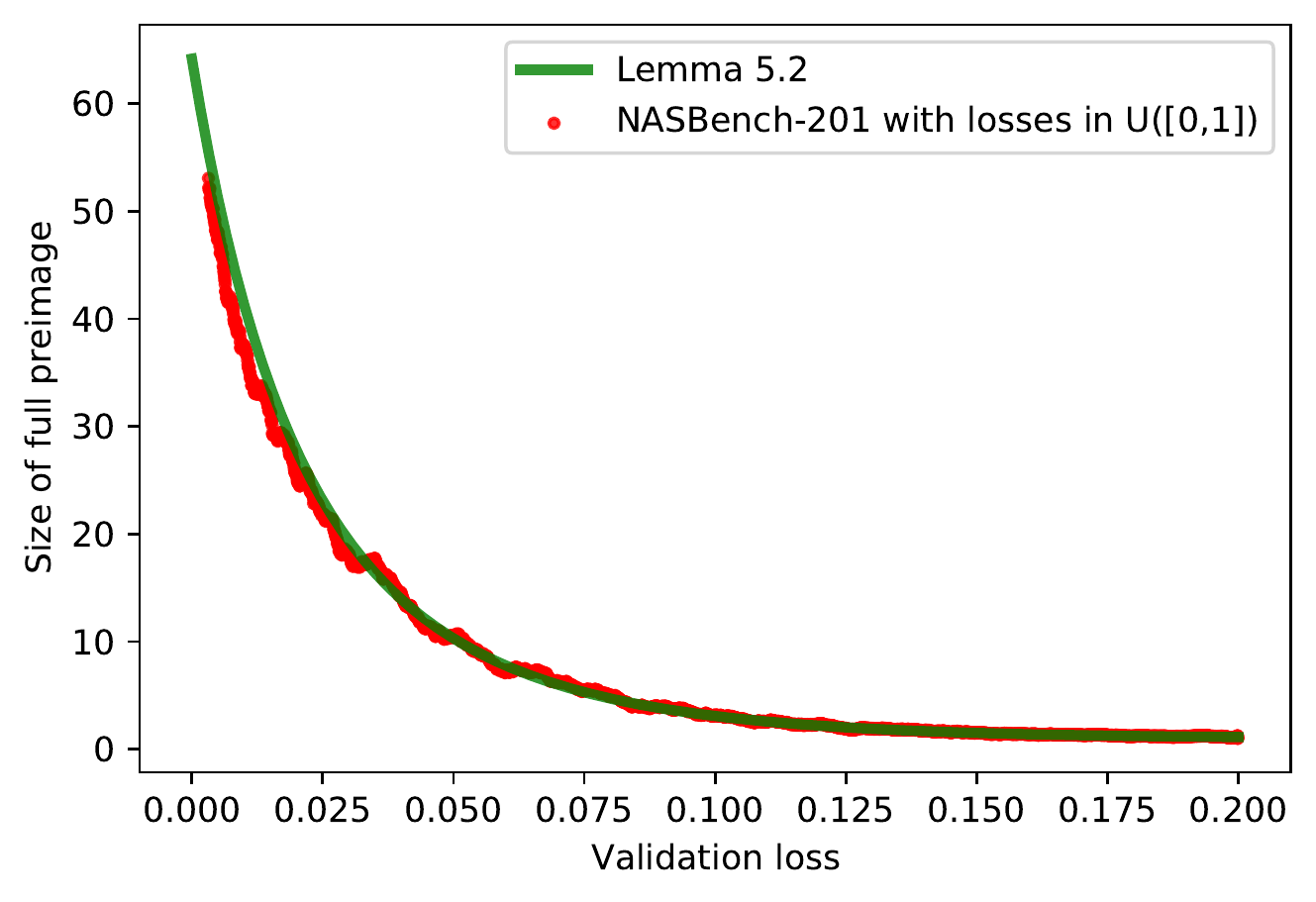}
\caption{
Performance of NAS algorithms on 
standard and denoised versions of NASBench-101 (top left/middle)
and NASBench-301 (bottom left/middle).
Probability that local search will converge to within $\epsilon$
of the global optimum, compared to Theorem~\ref{thm:prob_opt} (top right).
Validation loss vs.\ size of preimages, compared to Lemma~\ref{lem:gen_eqns} (bottom right).
}
\label{fig:ls_baselines_201}
\end{figure*}

\begin{figure*}
\centering % 
\includegraphics[width=0.33\textwidth]{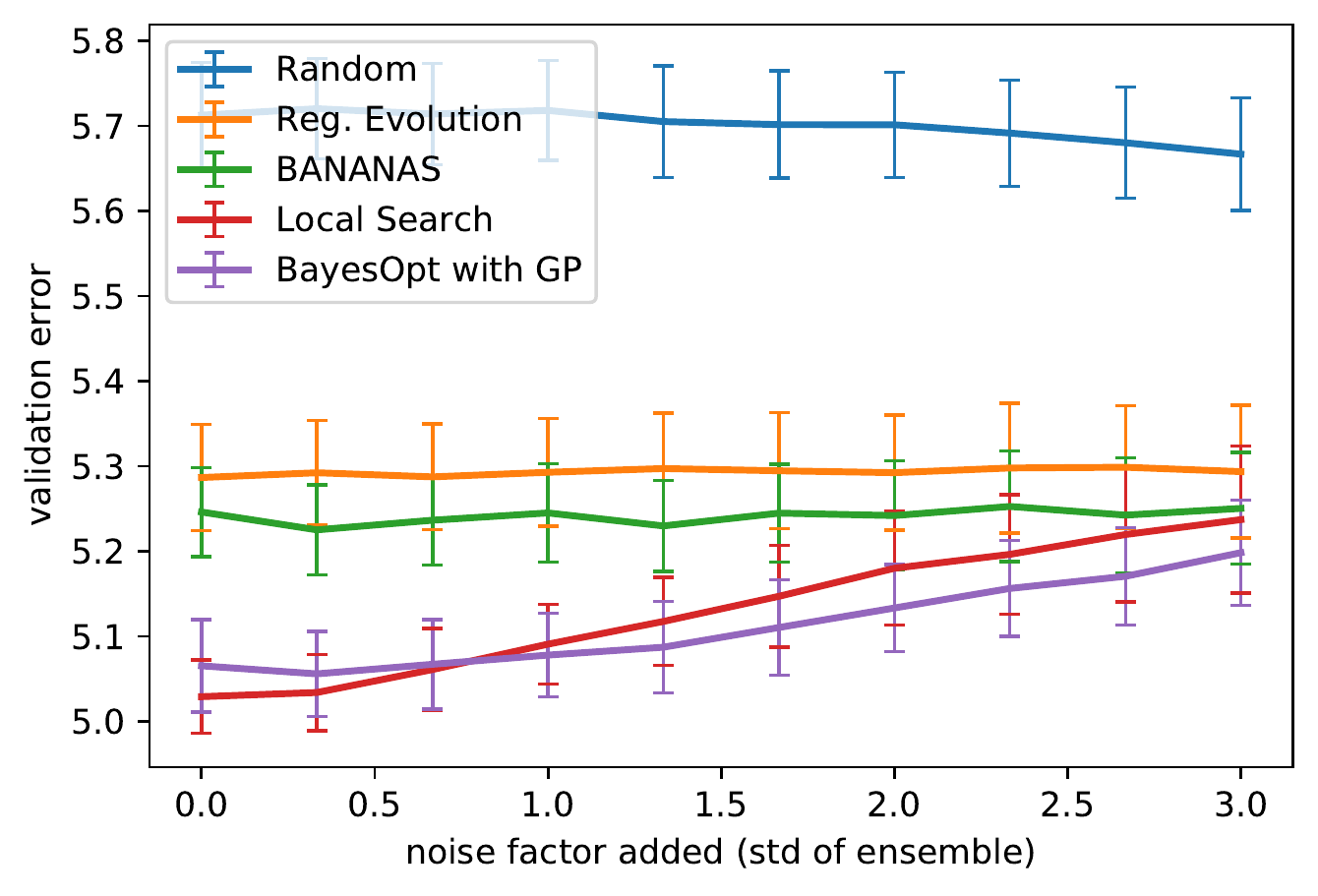}
\hspace{-3pt}
\includegraphics[width=0.33\textwidth]{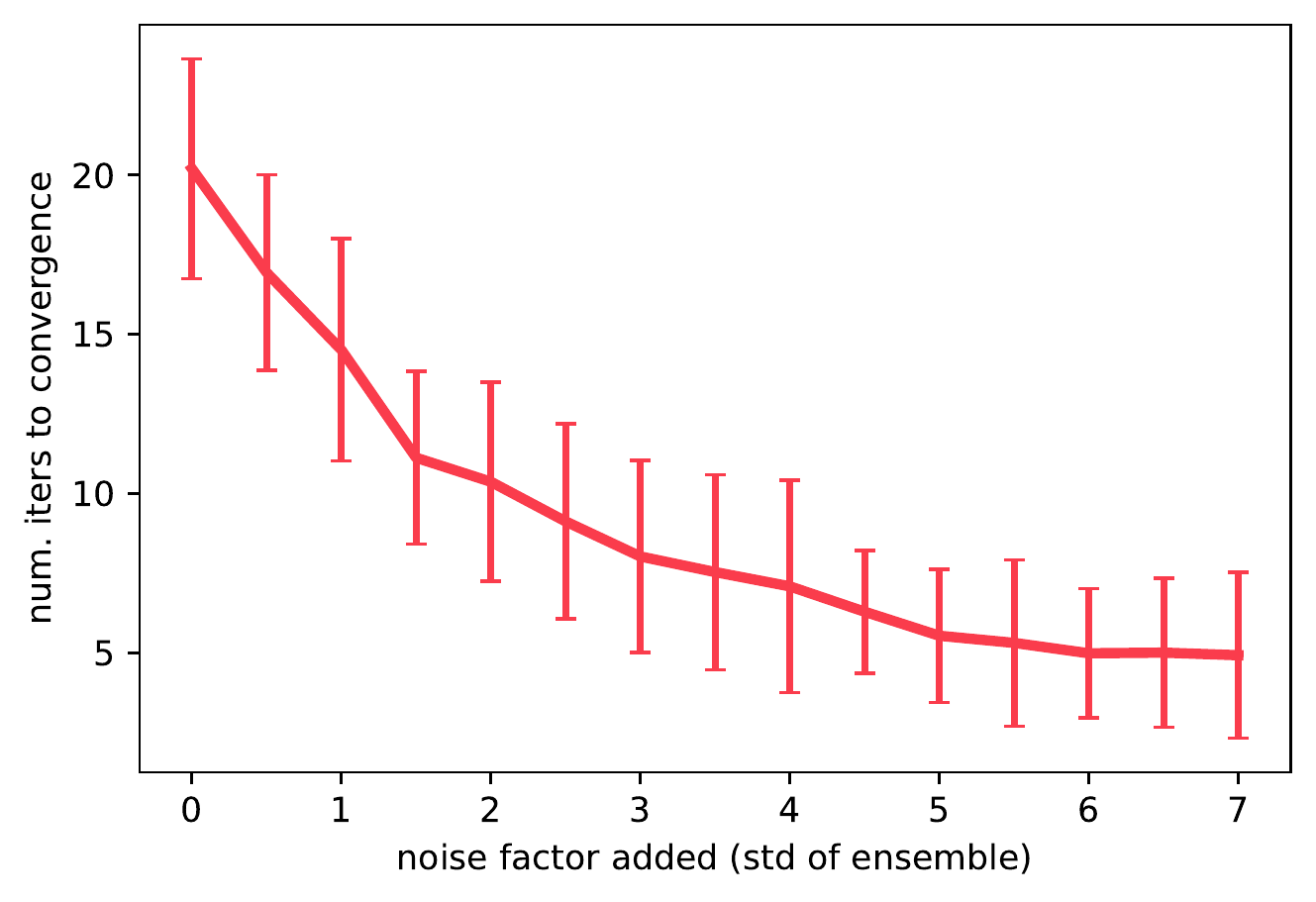}
\hspace{-3pt}
\includegraphics[width=0.33\textwidth]{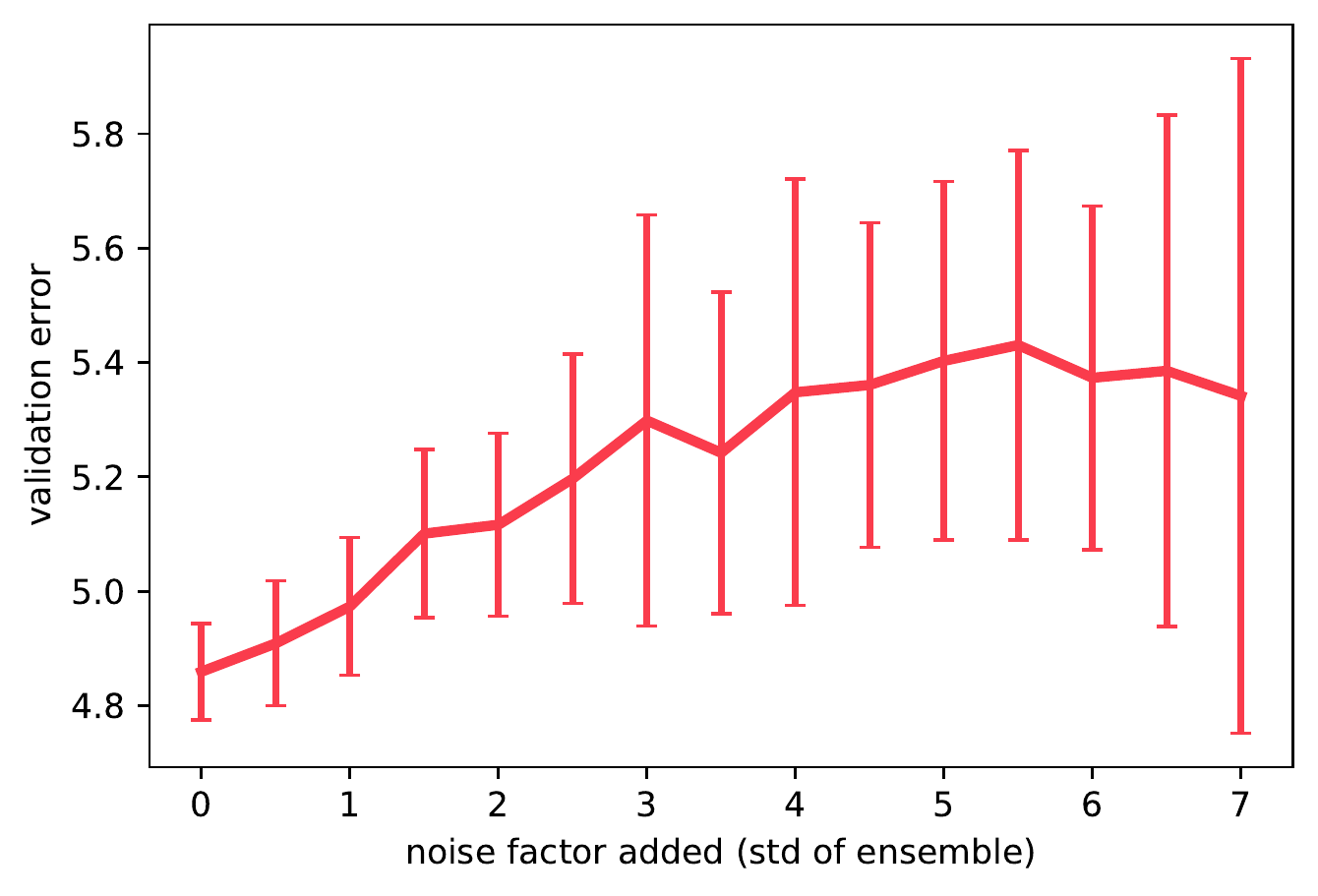}
\caption{
Amount of noise present in the architecture evaluation step of NASBench-301 
vs.\ performance of NAS algorithms (left), iterations (middle), 
and validation error (right).
}
\label{fig:301_trends}
\end{figure*}

\paragraph{Discussion.}
The simple local search algorithm achieves competitive performance on all NAS
benchmarks, beating out many popular algorithms.
Furthermore, we see a distinct trend across different benchmarks showing that
local search performs best (relative to other algorithms) 
when the noise in the training pipeline is reduced to a minimum.
Further experimentation shows that with less noise, there are fewer local minima
and local search takes more iterations to converge.
These results suggest that NAS becomes substantially easier when the noise is
reduced - enough for a very simple algorithm to achieve strong performance.
Since local search can be implemented in five lines of code, 
we encourage local search to be used as a benchmark in future work.
We also suggest denoising the noise in the training pipeline.
This can be achieved by techniques such as cosine annealing the 
learning rate~\citep{loshchilov2016sgdr}, batch normalization~\citep{ioffe2015batch},
and regularization techniques such as dropout~\citep{baldi2013understanding} 
and early-stopping~\citep{prechelt1998early}.

\section{Theoretical characterization} \label{sec:method}

In this section, we give a theoretical analysis of local search for NAS,
including a complete characterization of its performance.
We present a general result which can be applied to any NAS search space.
We also give an experimental validation of our results at the end of the
section, which suggests that our theoretical results predict the performance
of real datasets reasonably well.

In a NAS application,
the topology of the search space is fixed and discrete, 
%(e.g., $\left(K_5\right)^6)$,
while the distribution of validation losses
is randomized and continuous,
due to the non-deterministic nature of training a neural network.
Therefore, we assume that the validation loss for a trained architecture 
is sampled from
a global probability distribution, and for each architecture, 
the validation losses of its neighbors are sampled from a local probability distribution.
%
%The distribution for the validation loss of any architecture in the search space 
%is given by $\pdfn(x) \forall x \in \R$. 
%Given a validation loss $x\in \R$, the distribution for the validation loss of the neighbors
%of an architecture with validation loss $x$ is given by $\pdfe(x, y) \forall x, y \in \R$. 
%This gives us a probability over the validation loss for a neighboring architecture.
%
Recall the definitions of $G_N$ and $\ls$ from the end of
Section~\ref{sec:prelim}.
Given a graph $G_N=(A, E_N)$, each node $v \in A$ has a loss 
$\ell(v) \in \R$ sampled from a PDF which we denote by $\pdfn$. 
For any two neighbors $(v, u) \in E_N$, the PDF for the validation loss $x$ 
of architecture $u$ is given by $\pdfe(\ell(v), x)$.
Choices for the distribution $\pdfe$ are constrained by the 
fixed topology of the search space, as well as the distribution $\pdfn$.
% technically, to formally ground our theory, we assume E_N is a random variable
%In the supplementary material,
In Appendix~\ref{app:method}, 
we discuss this further by formally defining
measurable spaces for all random variables in our framework.

Our main result is a formula for the fraction of nodes in the search space which 
are local minima, as well as a formula for the fraction of nodes $v$ 
such that the loss of $\ls^*(v)$ 
is within $\epsilon$ of the loss of the global optimum, for all $\epsilon\geq 0.$
In other words, we give a formula for the probability that the local search algorithm 
outputs a solution that is close to optimal.
Note that such a formula characterizes the performance of local search.
We give the full proofs for all of our results 
%in the supplementary material.
in Appendix~\ref{app:method}.
For the rest of this section, 
we assume for all $v \in A$, $|N(v)|=s$, and
we assume $G_N$ is vertex transitive
(given $u,v\in A$, there exists an automorphism of $G_N$ which maps $u$ to $v$).
Let $v^*$ denote the architecture with the global minimum loss, therefore
the support of the distribution of validation losses is a subset of $[\ell(v^*),\infty).$ 
That is, $\int_{\ell(v)}^\infty \pdfn(v)dv=1.$ 
Technically, the integrals in this section are Lebesgue integrals. However, we use the more standard Riemann-Stieltjes notation for clarity. We also slightly abuse notation and
define $\ls^{-*}(v)=\ls^{-*}(x)$ when $\ell(v)=x.$
In the following statements, we assume there is a fixed graph $G_N$, and the 
validation accuracies are randomly assigned from a distribution defined by $\pdfn$ 
and $\pdfe$. Therefore, the expectations are over the random draws from $\pdfn$ and $\pdfe$.
\footnote{
In particular, given a node $v$ with validation loss $\ell(v)$
the probability distribution for the validation loss of a neighbor depends only 
on $\ell(v)$ and $\pdfe$, which makes the local search procedure similar to a 
Markov process.  Our experiments in Figure~\ref{fig:ls_baselines_201}
suggest this is a reasonable assumption in practice.
}

\begin{restatable}{rethm}{probopt}\label{thm:prob_opt}
Given $|A| = n$, $\ell$, $s$, $\epsilon$, $\pdfn$, and $\pdfe$,
\begin{align*}
&\E[|\{v \in A\mid \ls^*(v) = v\}|]\\
&=n \int_{\ell(v^*)}^\infty \pdfn(x) 
\left(\int_{x}^\infty \pdfe(x, y) dy \right)^s dx,\text{ and}
\end{align*}
\begin{align*}
&\E[|\{v \in A\mid \ell(\ls^*(v))-\ell(v^*)\leq\epsilon\}|]\\
&=n \int_{\ell(v^*)}^{\ell(v^*)+\epsilon}
\pdfn(x)\left(\int_{x}^\infty \pdfe(x,y) dy\right)^s\\
&\quad~\cdot\E[|\ls^{-*}(x)|]dx.
\end{align*}
\end{restatable}

\begin{proof}[\textbf{Proof sketch.}]

%Intuitively, to compute the probability that local search outputs a
%near-optimal solution, we first compute the probability a point $v$ is a local
%minimum by integrating $\pdfe(\ell(v), y)$ from $\ell(v)$ to $\infty$ over $y$
%and raising it to the power $|N(v)|=s.$
%We multiply this probability by the expected size of $v$'s full preimage to give us the size
%of the set of nodes that will converge to $v$.
%Then we integrate this probability over all values of $\ell(v)$ drawn from $\pdfn$.

To prove the first statement,
we introduce an indicator random variable %on the architecture space 
to test if the 
architecture is a local minimum: $I(v) = \ind\{\ls^*(v) = v\}.$ Then
\begin{align*}
&\E[|\{v \in A\mid \ls^*(v) = v\}|]\\
&= n \cdot P(\{v \in A \mid I(v) = 1\}) \\
%&= n\int_{\ell(v^*)}^{\infty} \pdfn(x)\cdot \P(\{x<\ell(u) \forall u\text{ s.t. }(u,v)\in E_N, x=\ell(v)\}) dx \\
&= n\int_{\ell(v^*)}^{\infty} \pdfn(x) \left(\int_{x}^{\infty} \pdfe(x,y) dy\right)^s dx.
\end{align*}
Intuitively, in the proof of the second statement, we follow similar reasoning but multiply
the probability in the outer integral by the expected size of $v$'s full preimage to weight
the integral by the probability a random point will converge to $v$.
Formally, we introduce an indicator random variable on the architecture space 
that tests if a node will terminate on a local minimum that is
within $\epsilon$ of the global minimum:
\begin{align*}
    I_\epsilon(v) &= \ind\{\ls^*(v) = u \land l(u) - l(v^*) \leq \epsilon\} \\
    %&= \ind\{\exists S \in \{\ls^{-*}(u): \ls^*(u) = u \land l(u) - l(v^*) \leq \epsilon\}, v \in S\}
\end{align*}
We use this random variable along with the first statement of the theorem,
to prove the second statement.
\begin{align*}
&\E[|\{v \in A \mid \ell(\ls^*(v))-\ell(v^*)\leq\epsilon\}|]\\
&= n \cdot P(\{I_\epsilon = 1\})\\
%&= n \int_{\ell(v^*)}^{\ell(v^*)+\epsilon}
%\P(\{v\in A \mid I(v) = 1, \ell(v) = x\})\E[|\ls^{-*}(x)|]dx\\
&= n\int_{\ell(v^*)}^{\ell(v^*)+\epsilon}\pdfn(x)
\left(\int_{\ell(v)}^\infty \pdfe(x, y) dy\right)^s\\
&\quad~\cdot\E[|\ls^{-*}(x)|]dx
\end{align*}
\end{proof}

In Appendix~\ref{app:method}, we use Theorem~\ref{thm:prob_opt}
along with Chebyshev's Inequality~\citep{chebyshev1867valeurs}
to show that, in the case where the validation accuracy of each architecture
has Gaussian noise, the expected number of local minima can be bounded in terms of
the standard deviation of the noise.
In the next lemma, we derive a recursive equation for $|\ls^{-*}(v)|.$
We define the \emph{branching fraction} of graph $G_N$ as $b_k=|N_k(v)|/\left(|N_{k-1}(v)|\cdot|N(v)|\right)$,
where $N_k(v)$ denotes the set of nodes which are distance $k$ to $v$ in $G_N$.
For example, the branching fraction of a tree with degree $d$ is $1$ for all $k$,
and the branching fraction of a clique is $b_1=1$ and $b_k=0$ for all $k>1.$
One more example is as follows.
%In the supplementary material,
In Appendix~\ref{app:experiments}, 
we show that the neighborhood graph of the
NASBench-201 search space is $(K_5)^6$ and therefore its branching factor is
$b_k=\frac{6-k+1}{6k}.$

\begin{restatable}{relem}{geneqns}\label{lem:gen_eqns}
Given $A$, $\ell$, $s$, $\pdfn$, and $\pdfe$,
then for all $v\in A$, we have the following equations.
\begin{align}
&\E[|\ls^{-1}(v)|]=s\int_{\ell(v)}^\infty \pdfe(\ell(v), y)\label{eq:e1}\\
&\quad~\cdot\left(\int_{\ell(v)}^\infty 
\pdfe(y,z)dz\right)^{s-1}dy,~\text{and}\nonumber
\end{align}
\begin{align}
&\E[|\ls^{-k}(v)|]=b_{k-1}\cdot \E[|\ls^{-1}(v)|]\label{eq:ex}\\
&\quad~\cdot\left(\frac{\int_{\ell(v)}^\infty \pdfe(\ell(v),y)\E[|\ls^{-(k-1)}(y)|]dy}
{\int_{\ell(v)}^\infty \pdfe(\ell(v),y)dy}\right).\nonumber
\end{align}
\end{restatable}

% short version: take out the above lemma and add:
%In Appendix~\ref{app:method}, we prove that
%$\E[|\ls^{-1}(v)|]=s\int_{\ell(v)}^\infty \pdfe(\ell(v), y)
%\left(\int_{\ell(v)}^\infty \pdfe(y,z)dz\right)^{s-1}dy$,
%and we give a recursive expression for $\E[|\ls^{-k}(v)|]$.
For some PDFs, it is not possible to find a closed-form solution
for $\E[|\ls^{-k}(v)|]$ because arbitrary 
functions may not have closed-form antiderivatives.
By assuming there exists a function $g$ such that $\pdfe(x,y)=g(y)$ for all $x$,
we can use induction to find a closed-form expression for $\E[|\ls^{-k}(v)|]$.
This includes the uniform distribution ($g(y)=1$ for $y\in [0,1]$),
as well as distributions that are polynomials in $x$.
%In the supplementary material,
In Appendix~\ref{app:method}, 
we use this to show that $\E[|\ls^{-*}(v)|]$
can be approximated by $1+s\cdot G(\ell(v))^s\cdot e^{G(\ell(v))^s},$
where $G(x)=\int_x^\infty g(y)dy.$
Now we use a similar technique to give a closed-form expression for Theorem~\ref{thm:prob_opt}
when the local and global distributions are uniform.
We stress that this lemma is simply an application of Lemma~\ref{lem:gen_eqns},
and our main results (Theorem~\ref{thm:prob_opt} and Lemma~\ref{lem:gen_eqns})
hold without any assumptions on the local and global distributions.

\begin{restatable}{relem}{fulluniform}\label{lem:full_uniform}
If $\pdfn(x)=\pdfe(x,y)=U([0,1])~\forall x\in A$,
then $\E[|\{v\mid v=\ls^*(v)\}|]=\frac{n}{s+1}$ and
\begin{align*}
&\E[|\{v\mid \ell(\ls^*(v))-\ell(v^*)\leq\epsilon\}|]\\
&=n \sum_{i=0}^\infty \left(\frac{s^i\left(1-(1-\epsilon)^{(i+1)s+1}\right)}{(i+1)s+1}
\cdot\prod_{j=0}^{i-1}\frac{b_j}{js+1}\right).
\end{align*}
\end{restatable}

\begin{proof}[\textbf{Proof sketch.}]
The probability density function of $U([0,1])$ is equal to 1 on $[0,1]$ and 0 otherwise.
Then $\int_{x}^\infty \pdfe(x,y)dy=\int_{x}^1 dy=(1-x).$
We use this in combination with Theorem~\ref{thm:prob_opt} to prove the first statement:
\begin{equation*}
\E[|\{v\mid v=\ls^*(v)\}|]
=n\int_{\ell(v^*)}^\infty 1\cdot 
\left(1-x\right)^s dx=\frac{n}{s+1}.
\end{equation*}

To prove the second statement, first
we use induction on the %recursive 
expression in Lemma~\ref{lem:gen_eqns}
to show that for all $v\in A$, 
\begin{align*}
&\E[|\ls^{-*}(v)|]=\sum_{k=0}^\infty \E[|\ls^{-k}(v)|]\\
&=\sum_{k=0}^\infty \left(s^k (1-\ell(v))^{sk} 
\cdot\prod_{i=0}^{k-1}\frac{b_i}{is+1}\right).
\end{align*}
We plug this into the second part of Theorem~\ref{thm:prob_opt}:
\begin{align*}
&\E[|\{v\mid \ell(\ls^*(v))-\ell(v^*)\leq\epsilon\}|]\\
&=n \int_{\ell(v^*)}^{\ell(v^*)+\epsilon}
1\cdot (1-x)^s \sum_{k=0}^\infty \E[|\ls^{-k}(x)|]dx\\
&=n \int_{\ell(v^*)}^{\ell(v^*)+\epsilon}
(1-x)^s \sum_{k=0}^\infty \left(s^k(1-x)^{sk} \prod_{i=0}^{k-1}\frac{b_j}{is+1}\right)dx\\
&=n \sum_{i=0}^\infty\left(\frac{s^i\left(1-(1-\epsilon)^{(i+1)s+1}\right)}{(i+1)s+1}
\cdot\prod_{j=0}^{i-1}\frac{b_j}{js+1}\right).
\end{align*}
\end{proof}

In the next section, we show that our theoretical results can be used to
predict the performance of local search.

% to cut down this section: 
% remove the last align env in the last proof
% remove the middle lemma and just describe it in words, but make sure it is consistent

\begin{figure}
\centering % 
\includegraphics[width=0.42\textwidth]{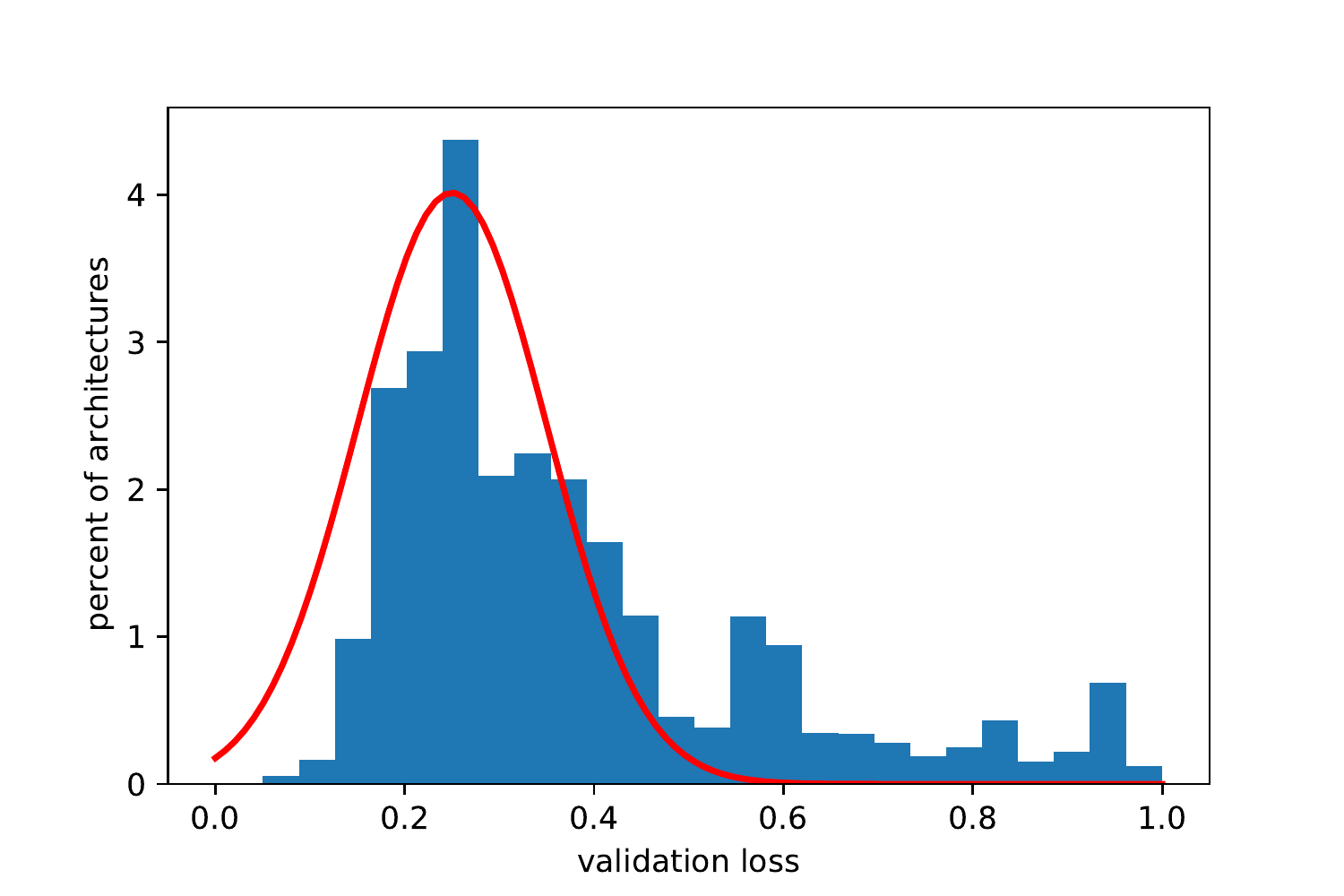}
\includegraphics[width=0.42\textwidth]{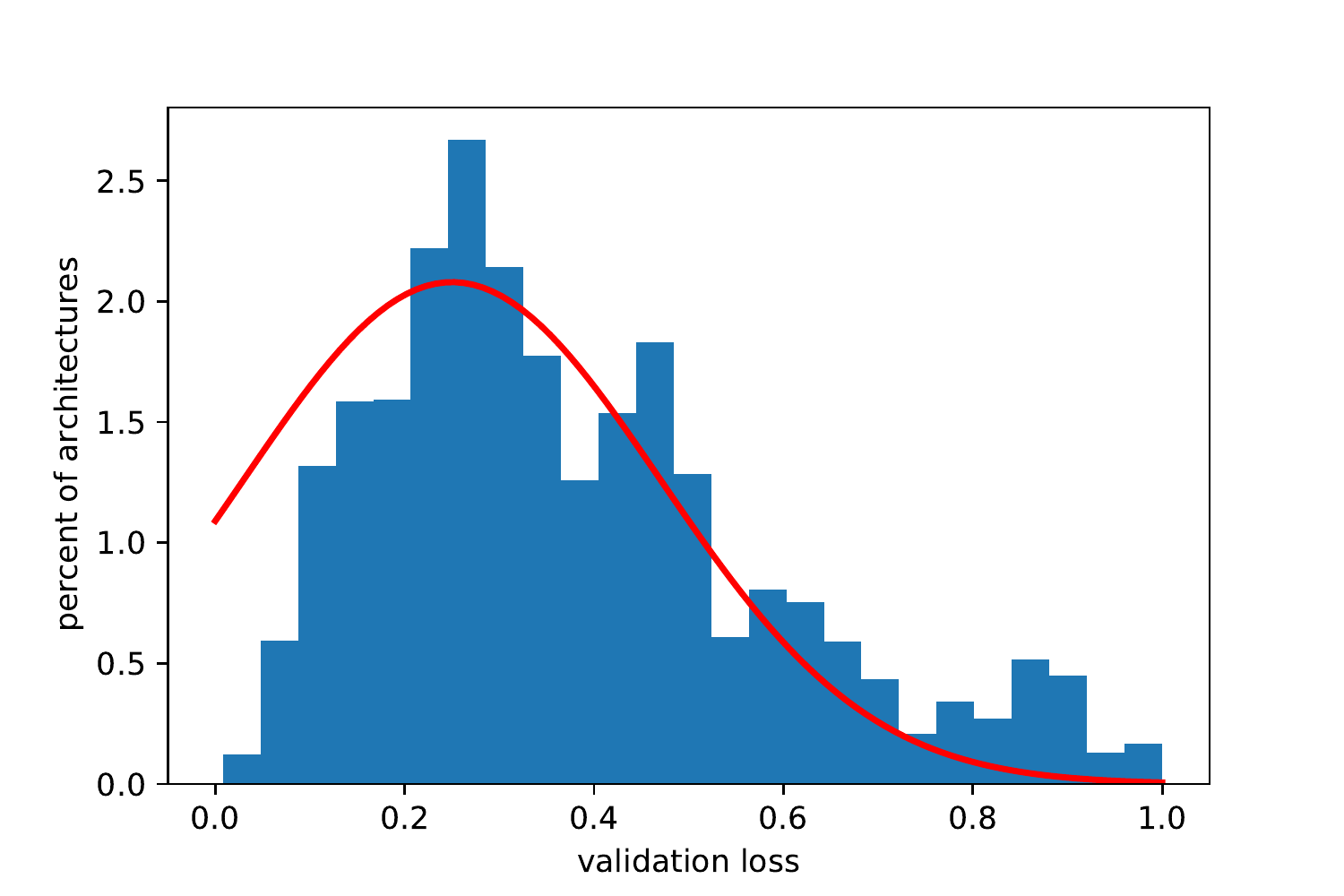}
\caption{
Histogram of validation losses for CIFAR-100 (top) and ImageNet16-120 (bottom) 
in NASBench-201,
fitted with the best values of $\sigma$ and $v$ in Equation~\ref{eq:normal_pdf}.
}
\label{fig:single_histogram}
\end{figure}

\begin{figure*}
\centering % 
\includegraphics[width=0.32\textwidth]{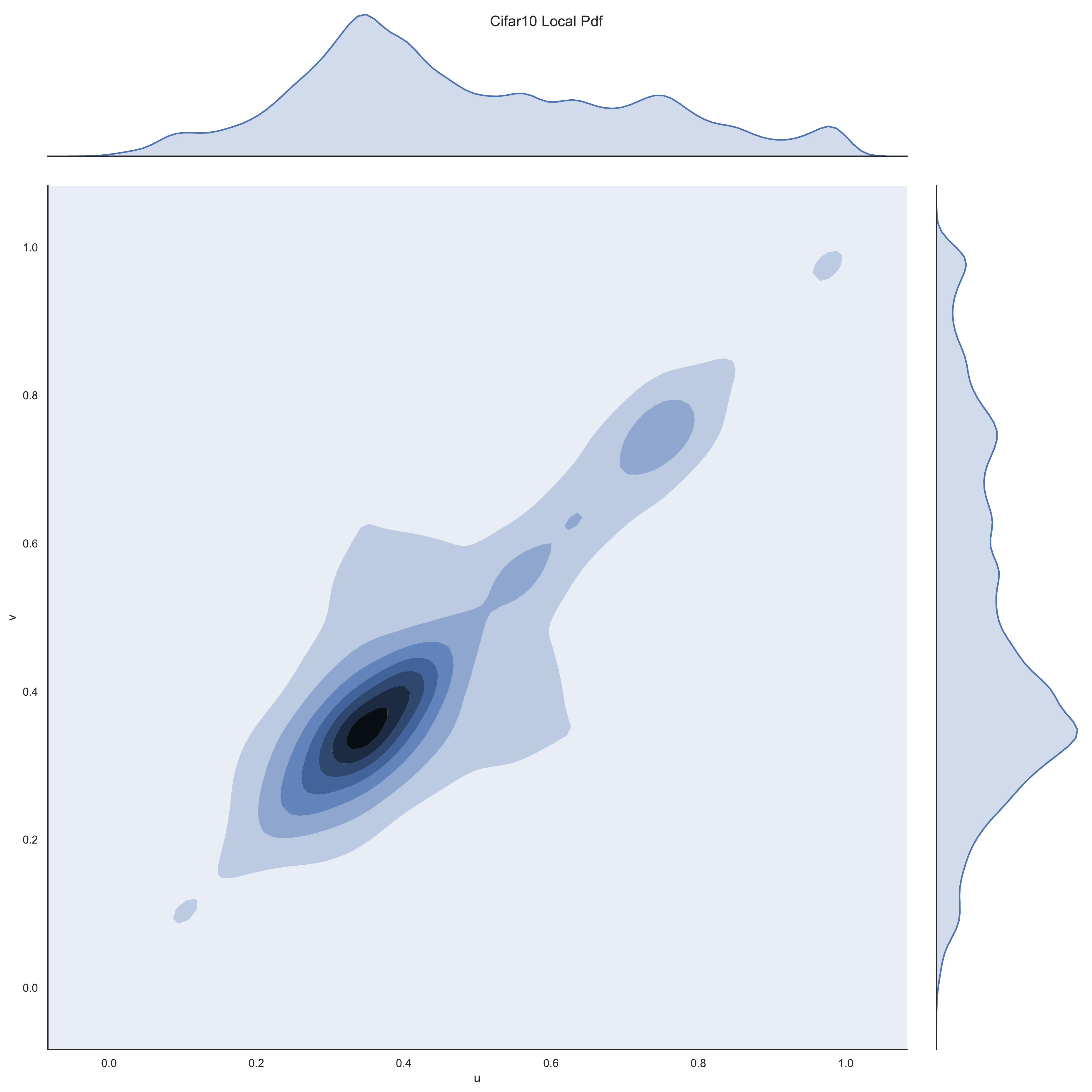}
%\hspace{-3pt}
\includegraphics[width=0.32\textwidth]{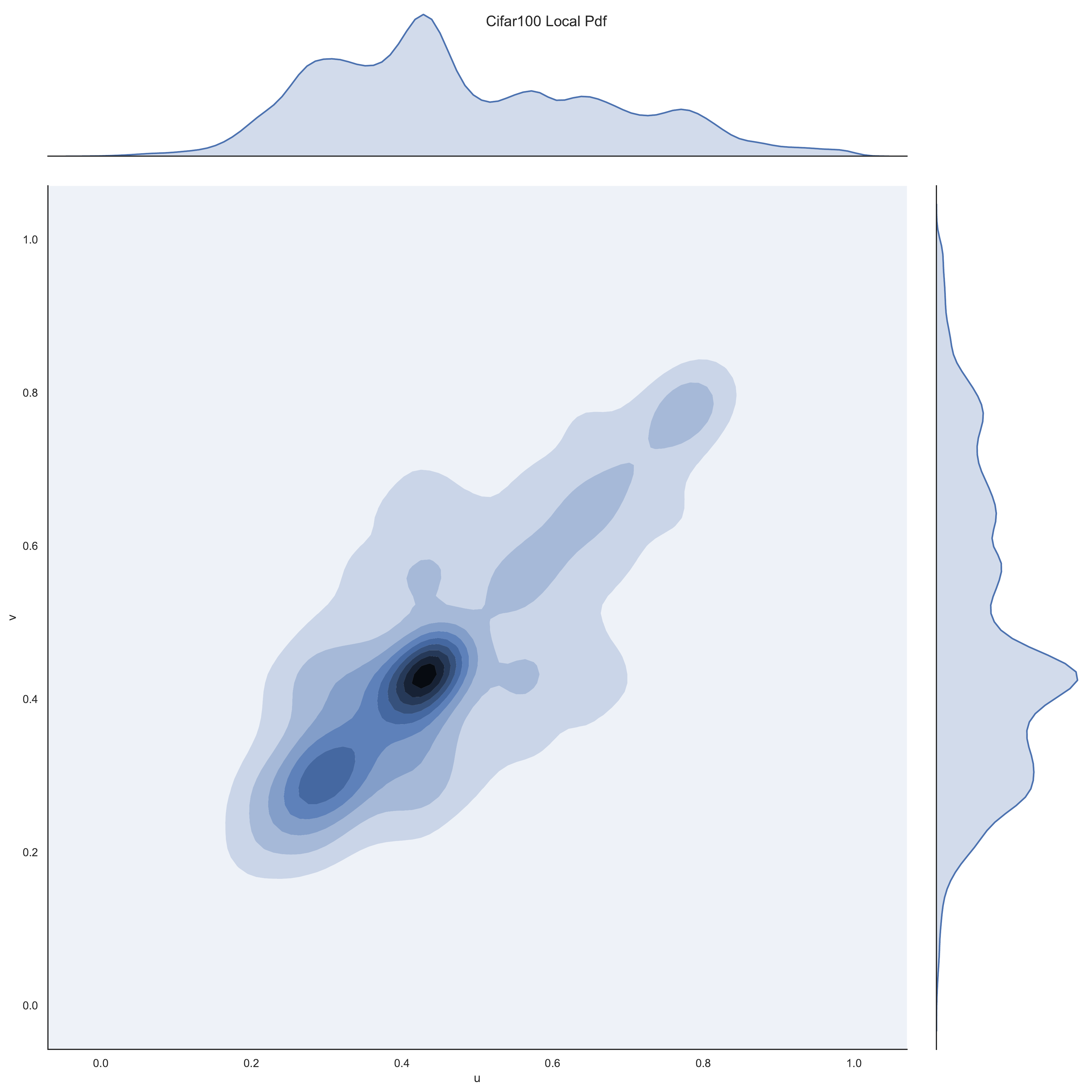}
%\hspace{-3pt}
\includegraphics[width=0.32\textwidth]{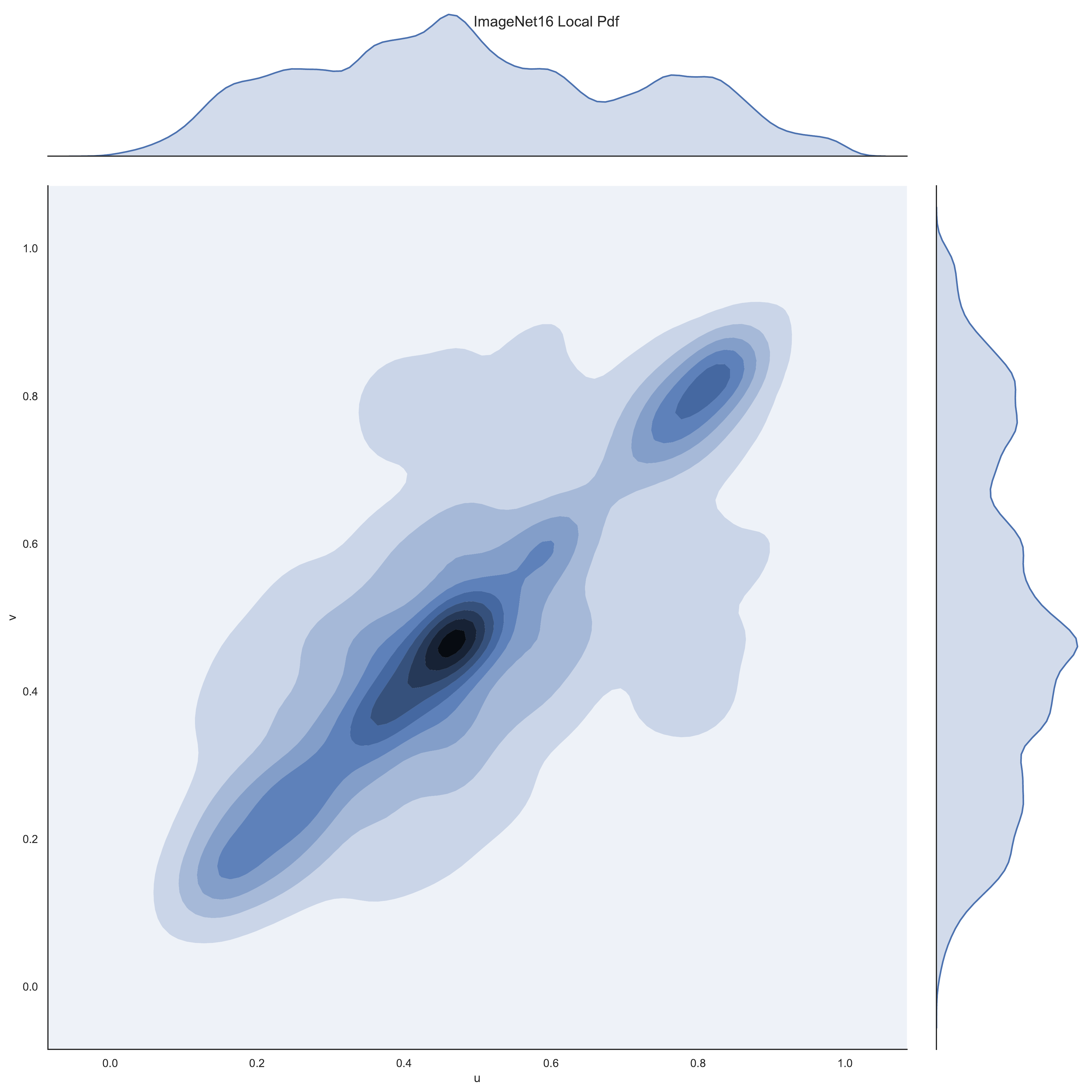}
\caption{
Probability density function for CIFAR-10, CIFAR-100, and
ImageNet16-120 on NASBench-201. For each coordinate $(u,v)$, a darker color indicates
that architectures with accuracy $u$ and $v$ are more likely to be neighbors.
}
\label{fig:201_local_pdfs}
\end{figure*}

\paragraph{Simulation Results.}
We run a local search simulation using the equations in
the previous section
as a means of experimentally validating our theoretical results
with real data (we use NASBench-201).
In order to use these equations, first we must approximate the local
and global probability density functions of the three 
datasets in NASBench-201.
We note that approximating these distributions are not feasible for large search
spaces; the purpose of our theoretical results are meant only to provide a deeper
understanding of local search and lay the groundwork for future studies.
We start by visualizing the probability density functions of the three datasets.
See Figure~\ref{fig:201_local_pdfs}.
We see the most density along the diagonal, meaning that architectures
with similar accuracy are more likely to be neighbors.
Therefore, we can approximate the PDFs by using the following equation:

%
%The true local PDF of the NASBench-201 is multimodal 
%(see Appendix~\ref{app:experiments}),
%(see Figure~\ref{fig:201_local_pdfs}),
%Therefore, we focus on modeling the dataset with simpler, one-parameter functions 
%which will be easier to estimate for new datasets:

\begin{equation}
\pdf(u)=\frac{
\frac{1}{\sigma\sqrt{2\pi}}\cdot 
e^{-\frac{1}{2}\left(\frac{u-v}{\sigma}\right)^2}
}
{
\int_0^1 \frac{1}{\sigma\sqrt{2\pi}}\cdot 
e^{-\frac{1}{2}\left(\frac{w-v}{\sigma}\right)^2}dw
}
\label{eq:normal_pdf}
\end{equation}

This is a normal distribution with mean $u-v$ 
and standard deviation $\sigma$, truncated so that it is a valid PDF
in $[0,1].$
We note that prior work has also modeled architecture accuracies in NAS with
a normal distribution~\citep{real2019regularized}.
To model the global PDF for each dataset,
we plot a histogram of the validation losses and match them to the closest-fitting
values of $\sigma$ and $v$.
See Figure~\ref{fig:single_histogram}.
The best values of $\sigma$ are $0.18,~0.1,$ and $0.22$ for CIFAR-10, CIFAR-100, 
and ImageNet16-120, respectively, and the best values for $v$ are all $0.25$.
%\footnote{
%Note that plotting a histogram of all validation losses is impractical 
%for real-world NAS search spaces; we do this on NASBench-201 as a means 
%of validating our theoretical results.
%}
To model the local PDF for each dataset, we compute
the random walk autocorrelation (RWA) on each dataset.
RWA is defined as the autocorrelation of the accuracies of points visited during a
random walk on the neighborhood 
graph~\citep{weinberger1990correlated, stadler1996landscapes},
and was used to measure locality in NASBench-101 in prior work~\citep{nasbench}.
For the full details of the steps taken to model the datasets in NASBench-201, 
%see the supplementary material.
see Appendix~\ref{app:experiments}.

%\paragraph{Results.}
Now %that we have chosen the best parameters for the local and global PDFs,
we use Theorem~\ref{thm:prob_opt} to compute the 
probability that a randomly drawn architecture will converge to within $\epsilon$
of the global minimum when running local search.
Since there is no closed-form solution for the expression in Lemma~\ref{lem:gen_eqns},
we compute Theorem~\ref{thm:prob_opt} up to the 5th preimage.
We compare this to the experimental results on NASBench-201.
We also compare the performance of the NASBench-201 search space with validation losses drawn uniformly at random, 
to the performance predicted by Lemma~\ref{lem:full_uniform}.
%We drew a random validation loss in $[0,1]$ for all 15625 architectures in the
%NASBench-201 search space, and ran local search from each architecture.
Finally, we compare the preimage sizes of the architectures in NASBench-201 with
randomly drawn validation losses to the sizes predicted in Lemma~\ref{lem:gen_eqns}.
%Appendix~\ref{app:method} or Lemma~\ref{lem:exindependent}.
See Figure~\ref{fig:ls_baselines_201}.
Our theory exactly predicts the performance and the preimage sizes of the uniform 
random NASBench-201 dataset.
On the three image datasets, our theory predicts the performance
fairly accurately, but is not perfect due to our assumption that the distribution
of accuracies is unimodal.

\section{Conclusion} \label{sec:conclusion}

We show that the difficulty of NAS scales dramatically with the
level of noise in the architecture evaluation pipeline, on popular NAS
benchmarks (NASBench-101, 201, and 301). 
In particular, the simplest local search algorithm is sufficient to outperform
popular state-of-the-art NAS algorithms when the noise in the evaluation pipeline
is reduced to a minimum. We further show that as the noise increases, the number
of local minima increases, and the basin of attraction to the global minimum
shrinks.
This suggests that when the noise in popular NAS benchmarks is reduced to a minimum, 
the number of local minima decreases, making
the loss landscape easy to traverse.
Since local search is a simple technique that often gives competitive 
performance, we encourage local search to be used as a benchmark for NAS in the future.
We also suggest denoising the training pipeline whenever possible in future NAS
applications.

Motivated by our findings, we give a theoretical study which explains the 
performance of local search for NAS on different search spaces.
We define a probabilistic graph optimization framework to study NAS problems, and
we give a characterization of the performance of local search for NAS in our framework. 
Our results improve the theoretical understanding of local search and
lay the groundwork for future studies.
%Our theoretical results may be of independent interest.
We validate this theory with experimental results.
Investigating more sophisticated variants of local search for NAS such as 
Tabu search, simulated annealing,
or multi-fidelity local search, are interesting next steps.

%    We give a theoretical characterization of the properties of a dataset 
%    necessary for local search to give strong performance. 
%    We experimentally validate these results on real neural architecture
%    search datasets.
%    Our results improve the theoretical understanding of local search and
%    lay the groundwork for future studies.

%\section{Broader Impact} \label{sec:impact}
%\input{impact}

\section*{Acknowledgments}
This work done while all authors were employed at Abacus.AI.
We thank Willie Neiswanger for his help with this project.

%\newpage

\bibliography{white_242}

%\begin{comment}
\clearpage
\newpage

\appendix

% weird, labels don't work inside section headers

%\section{Details from Section~\ref{sec:experiments}} \label{app:experiments}
\section{Details from Section~4} \label{app:experiments}

In this section, we give details and supplementary results for Section~\ref{sec:experiments}.

For every benchmark NAS algorithm, we used the code directly from its corresponding
open-source repository.
For regularized evolution, we changed the population
size from 50 to 30 to account for fewer queries.
We did not change any hyperparameters for the other baseline algorithms.
For vanilla Bayesian optimization, we used the ProBO 
implementation~\citep{neiswanger2019probo}.
Our experimental setup is the same as prior work (e.g.,~\citep{nasbench}).
At each timestep $t$, we report the test error of the architecture with the best
validation error found so far, and we run 200 trials of each algorithm and
average the result.

Now we give the local search experimental results for all three datasets
of NASBench-201. This is a similar experimental setup to the plots in
Figure~\ref{fig:ls_baselines_201} (left and middle), but for NASBench-201.
See Figure~\ref{fig:201_real}. 
We tested the simplest local search algorithm (hill-climbing),
as well as the $\cont$ variant described in Section~\ref{sec:prelim},
which we denote by Local++.
Note that for the case of ImageNet16-120,
the initial level of noise is so high that all NAS algorithms 
actually perform worse in the reduced noise version of the problem.

\begin{figure*}
\centering % 
\includegraphics[width=0.33\textwidth]{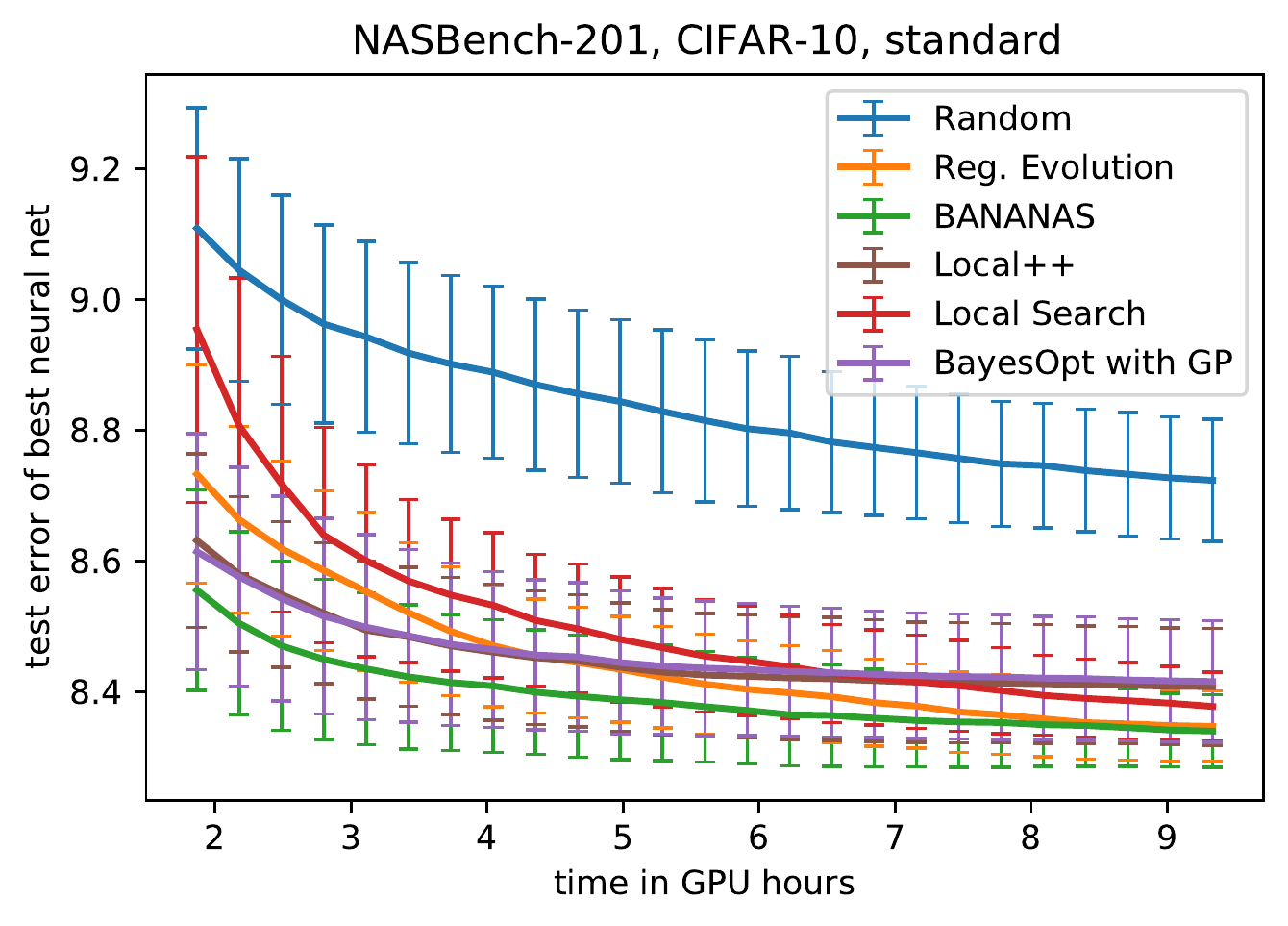}
\hspace{-3pt}
\includegraphics[width=0.33\textwidth]{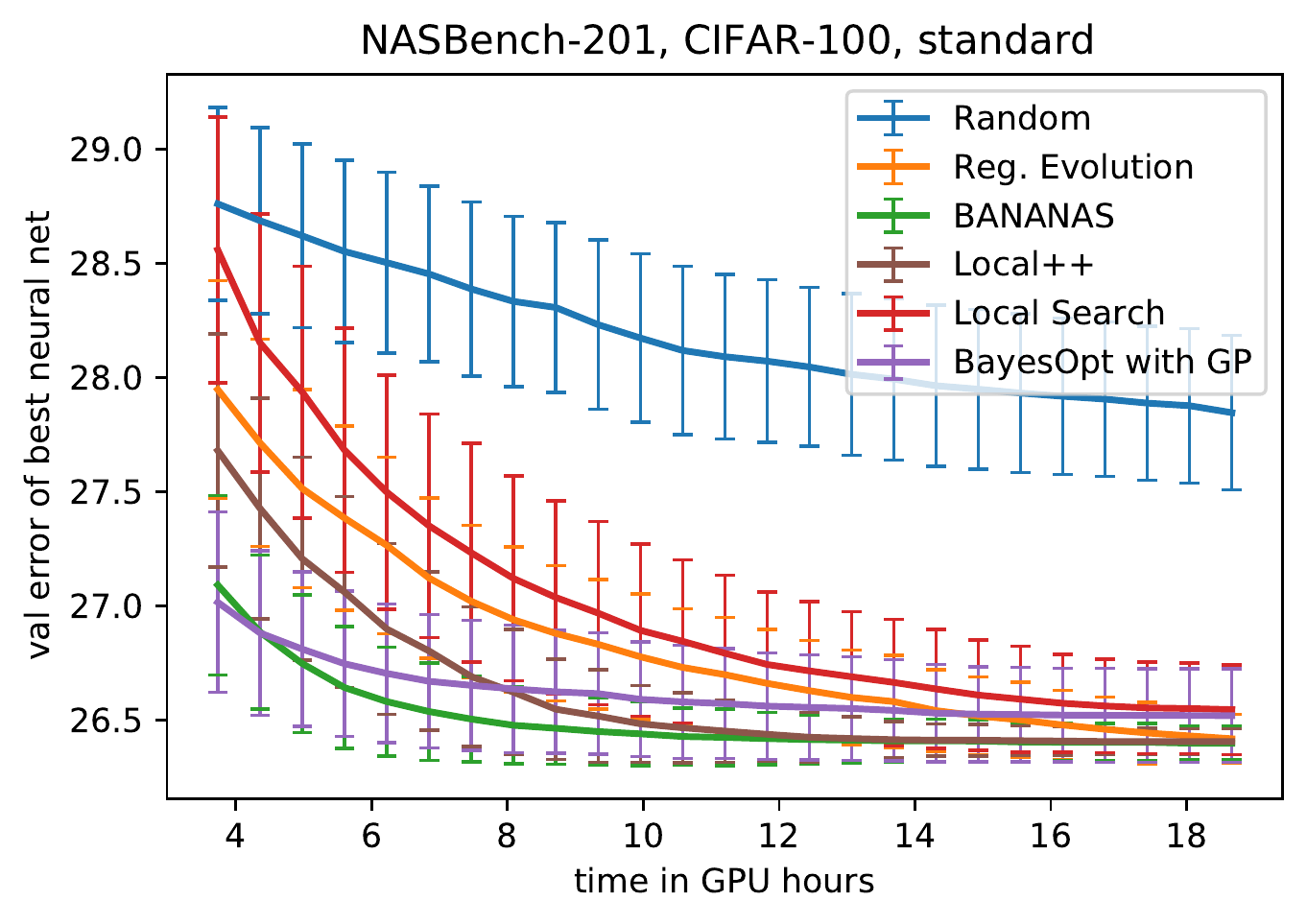}
\hspace{-3pt}
\includegraphics[width=0.33\textwidth]{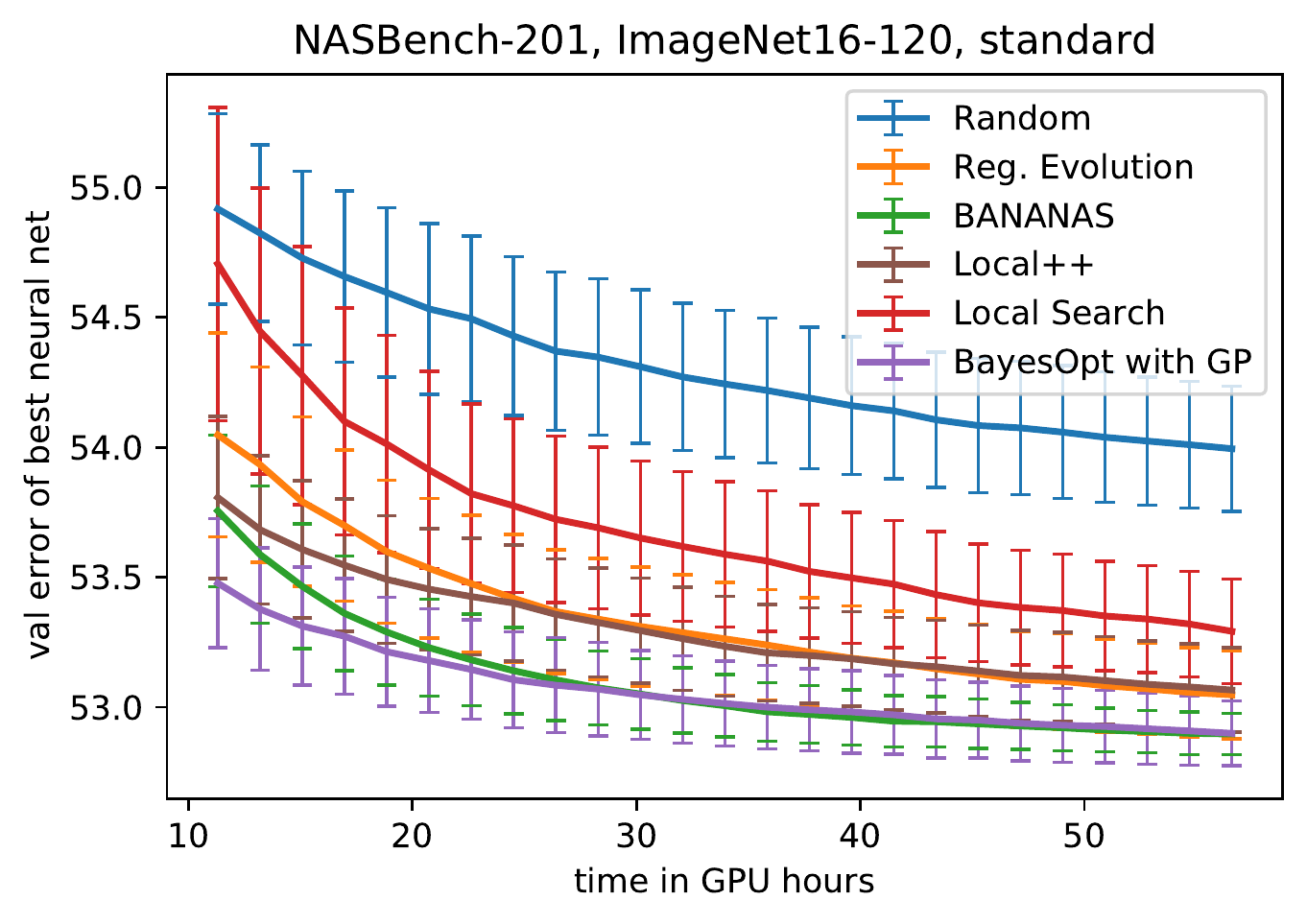}
\includegraphics[width=0.33\textwidth]{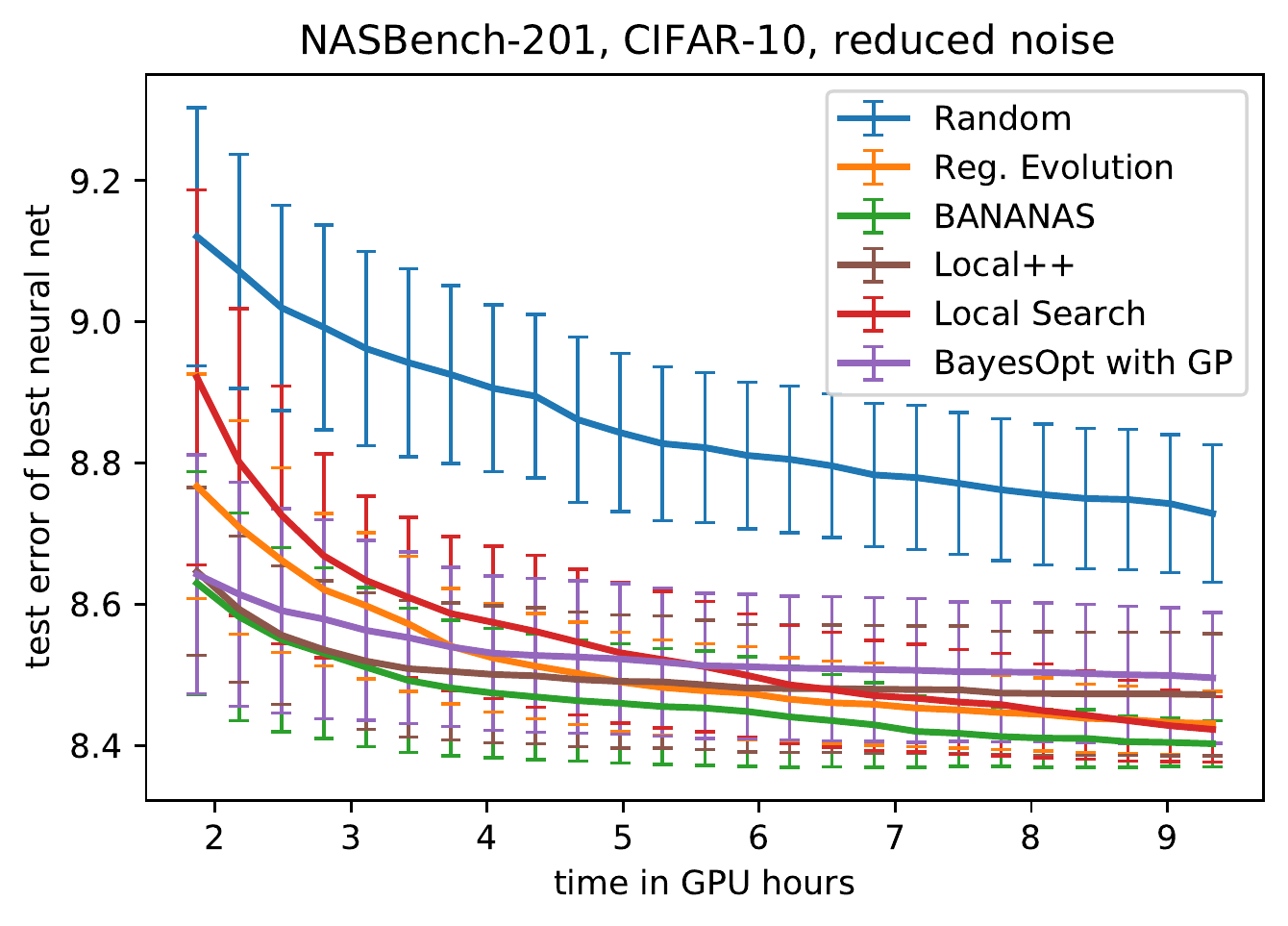}
\hspace{-3pt}
\includegraphics[width=0.33\textwidth]{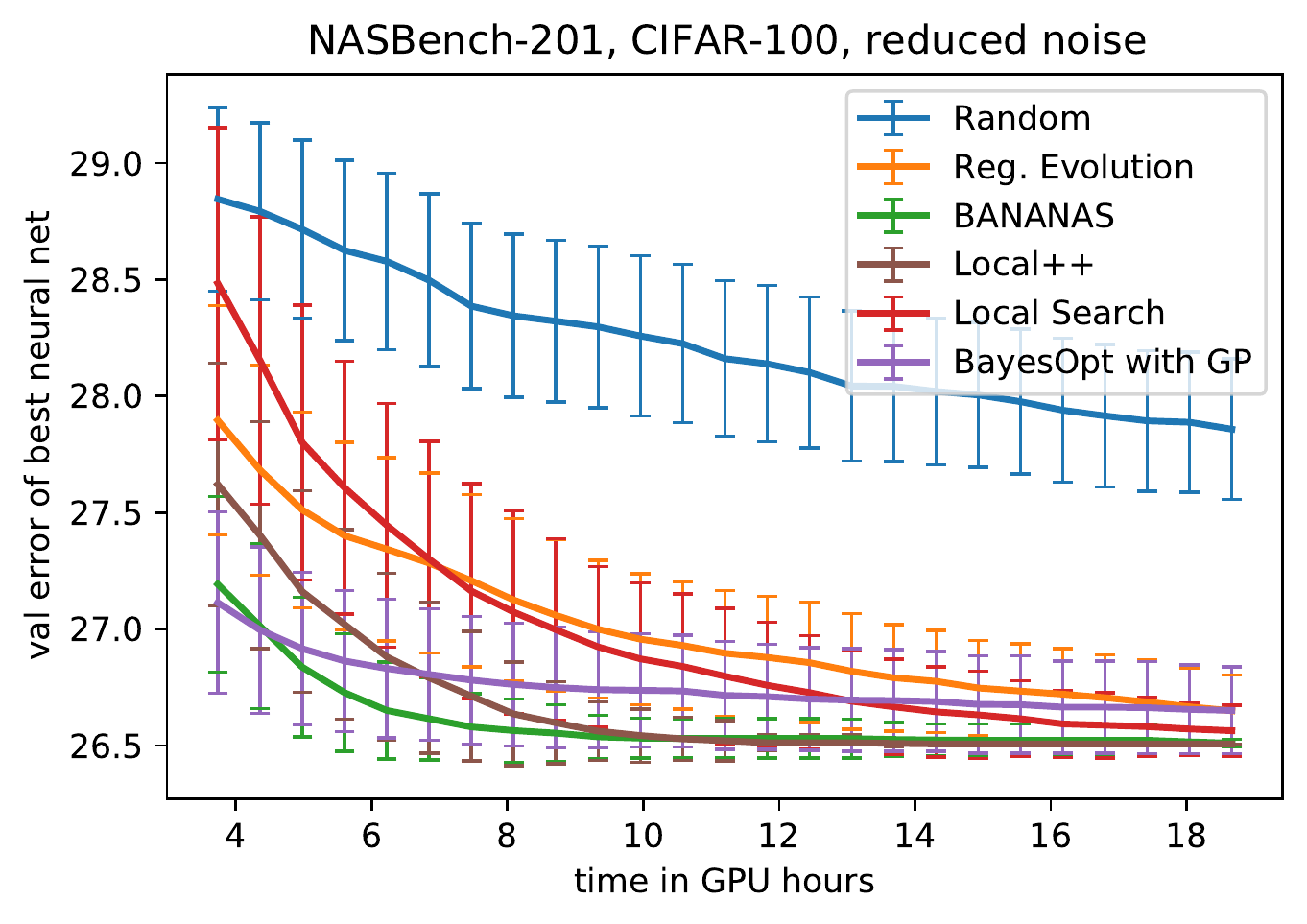}
\hspace{-3pt}
\includegraphics[width=0.33\textwidth]{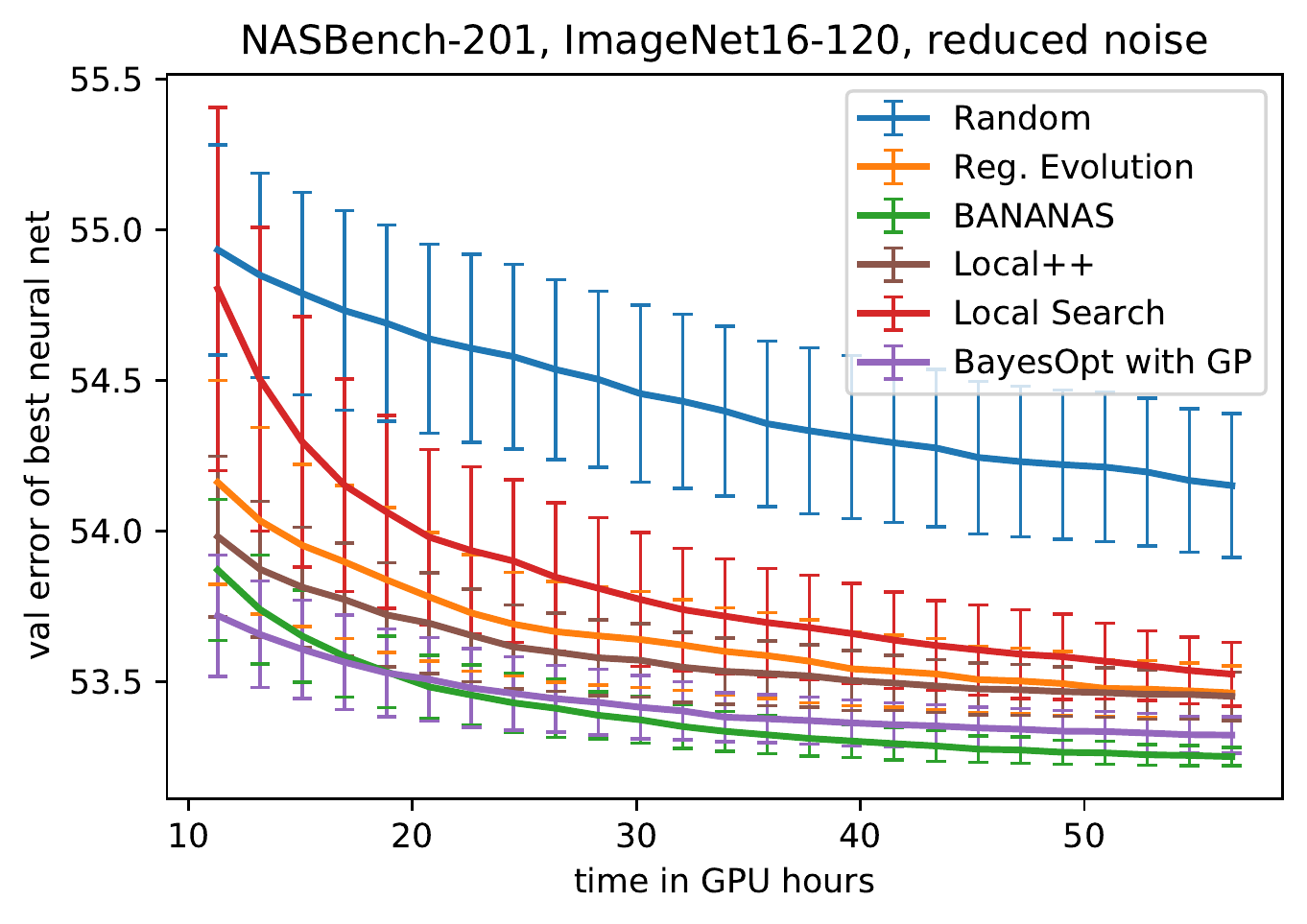}
\caption{
Performance of NAS algorithms on 
denoised (top) and standard (bottom) versions of NASBench-201 
CIFAR-10 (left), CIFAR-100 (middle), and ImageNet16-120 (right).
}
\label{fig:201_real}
\end{figure*}

% random init experiment
Now we evaluate the performance of local search as a function of the number
of initial random architectures drawn at the beginning.
We run local search with the number of initial random architectures set to 1, 
and 10 to 100 in increments of 10. For each number of initial random architectures,
we ran 2000 trials and averaged the results.
See Figure~\ref{fig:ls_init_201}.

\begin{figure*}
\centering % 
\includegraphics[width=0.98\textwidth]{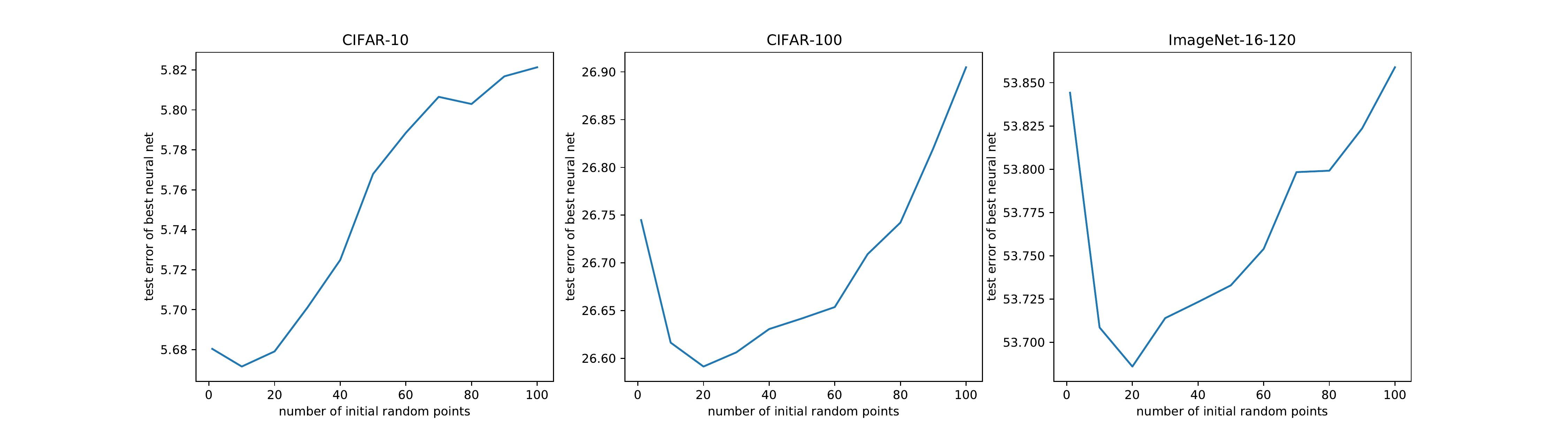}
\caption{
Results for local search performance vs.\ number of inital randomly drawn architectures
on NASBench-201 for CIFAR-10 (left), CIFAR-100 (middle),
and ImageNet-16-120 (right).
}
\label{fig:ls_init_201}
\end{figure*}

\subsection{Details from simulation experiments}

In this section, we give more details for our simulation experiment
described in Section~\ref{sec:experiments}.

For convenience, we restate Equation~\ref{eq:normal_pdf}, the function used to approximate
the datasets in NASBench-201.
\begin{equation*}
\pdf(u)=\frac{
\frac{1}{\sigma\sqrt{2\pi}}\cdot 
e^{-\frac{1}{2}\left(\frac{u-v}{\sigma}\right)^2}
}
{
\int_0^1 \frac{1}{\sigma\sqrt{2\pi}}\cdot 
e^{-\frac{1}{2}\left(\frac{w-v}{\sigma}\right)^2}dw
}
\end{equation*}
This is a normal distribution with mean $u-v$ 
and standard deviation of $\sigma$, truncated so that it is a valid PDF
in $[0,1].$
For a visualization, see Figure~\ref{fig:norm_pdf}.
In order to choose an appropriate probability density function for modelling the
datasets in NASBench-201, we approximate the $\sigma$ values for
both the local and global PDFs.

%\paragraph{Global Distribution.}
To model the global PDF for each dataset,
we plot a histogram of the validation losses and match them to the closest-fitting
values of $\sigma$ and $v$.
See Figure~\ref{fig:single_histogram}.
The best values are $\sigma=0.18,~0.1,$ and $0.22$ for CIFAR-10, CIFAR-100, 
and ImageNet16-120, respectively.

\begin{figure*}
\centering % 
\includegraphics[width=0.98\textwidth]{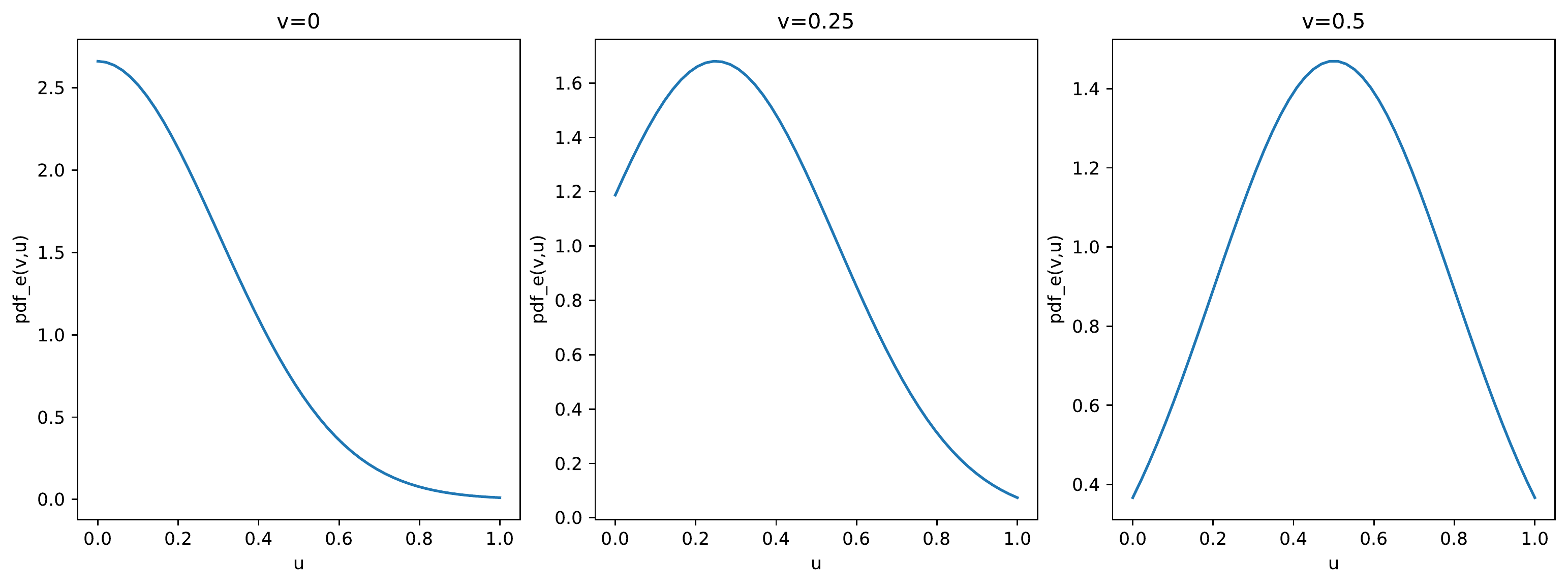}
\caption{
Normal PDF from Equation~\ref{eq:normal_pdf} plotted with three values of $v$.
}
\label{fig:norm_pdf}
\end{figure*}

\begin{figure*}
\centering % 
\includegraphics[width=0.4\textwidth]{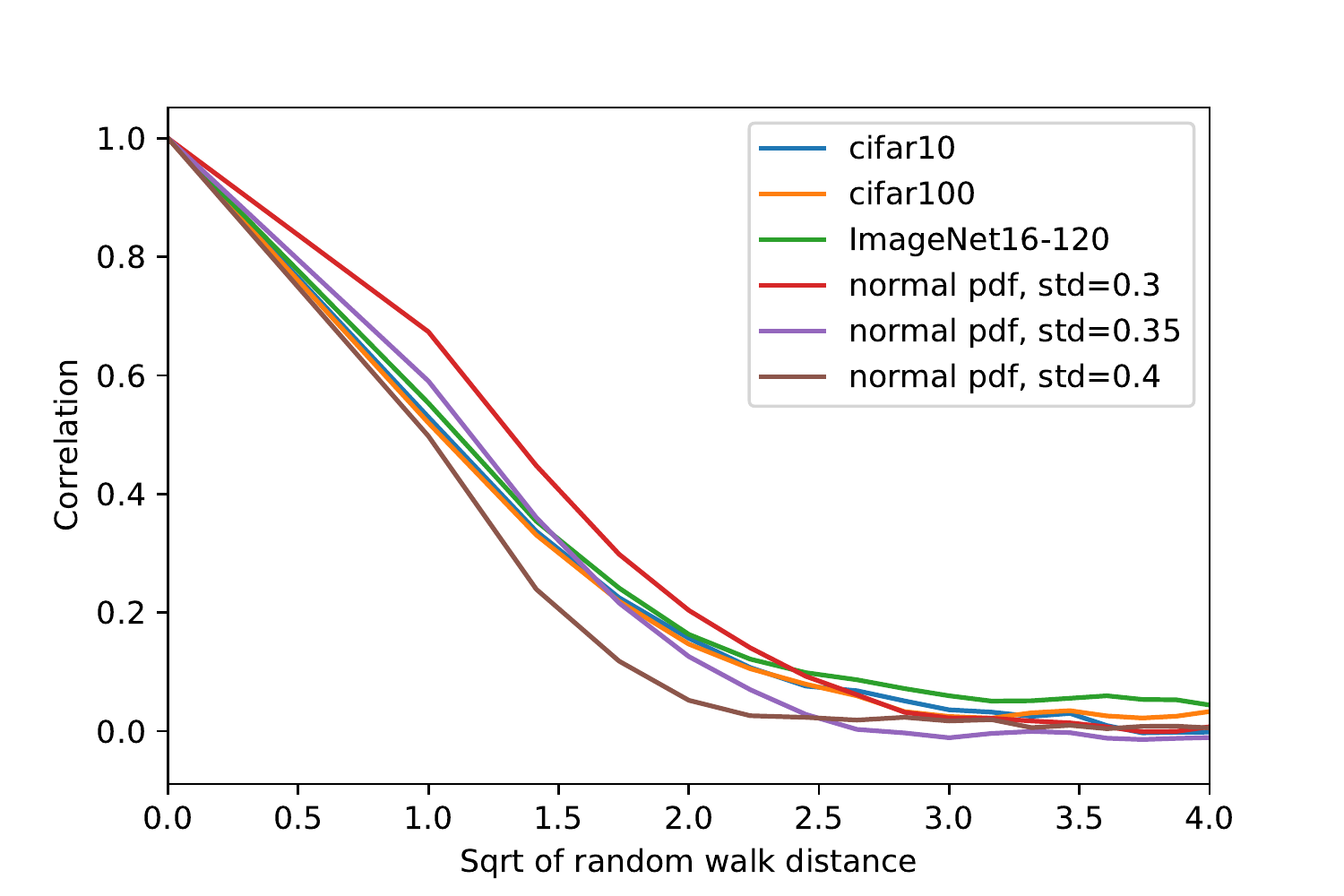}
\caption{
RWA vs.\ distance for three datasets in NASBench-201, as well as three values of 
$\sigma$ in Equation~\ref{eq:normal_pdf}.
Since a random walk reaches a mean distance of $\sqrt{N}$ after $N$ steps,
we plot the $x$-axis as the square root as the autocorrelation
shift, similar to prior work~\cite{nasbench}. 
}
\label{fig:rwa}
\end{figure*}

Now we plot the random-walk autocorrelation (RWA) described in Section~\ref{sec:experiments}.
Recall that RWA is defined as the 
autocorrelation of the accuracies of points visited during a
walk of random single changes through the 
search space~\citep{weinberger1990correlated, stadler1996landscapes},
and was used to measure locality in NASBench-101 in prior work~\citep{nasbench}.
We compute the RWA for all three datasets in NASBench-201, by performing a random
walk of length 100,000.
See Figure~\ref{fig:rwa}.
We see that all three datasets in NASBench-201, as evidenced because there is a
high correlation at distances close to 0.
As the diameter of NASBench-201 is 6, the correlation approaches zero at distances 
beyond about 3.5.
In order to model the local pdfs of each dataset,
we also compute the RWA for Equation~\ref{eq:normal_pdf}, and match each dataset
with the closest value of $\sigma$.
We see that a value of $\sigma=0.35$ is the closest match for all three datasets.

Now for each of the three NASBench-201 datasets, we have estimates for 
the $\pdfe$ and $\pdfn$
distributions. We plug each ($\pdfe$, $\pdfn$) pair into Theorem~\ref{thm:prob_opt},
which gives a plot of $\epsilon$ vs.\ percent of architectures that converge to within
$\epsilon$ of the global optimum after running local search.
We compare these to the true plot in Figure~\ref{fig:ls_baselines_201}.
For the random simulation, we are modeling the case where $\pdfe=\pdfn=U([0,1])$,
so we can use Lemma~\ref{lem:full_uniform} directly.

\subsection{Best practices for NAS research}
The area of NAS research has had issues with reproducibility and fairness in empirical
comparisons~\citep{randomnas, nasbench}, and there is now a checklist for best 
practices~\citep{lindauer2019best}. In order to promote best practices, we discuss each
point on the list, and encourage all NAS research papers to do the same.

\begin{itemize}
    \item \textbf{Releasing code.}
    %Our code is available in the supplementary material now, and will be made public after the double-blind review process.
    Our code is publicly available at\\ \url{https://github.com/naszilla/naszilla}.
    The training pipelines and search spaces are from popular existing NAS work:
    NASBench-101, NASBench-201, and NASBench-301.
    \item \textbf{Comparing NAS methods.}
    We made fair comparisons due to our use of NASBench-101, 201, and 301.
    For baseline comparisons, we used open-source code, a few times adjusting hyperparameters
    to be more appropriate for the search space.
    We ran ablation studies, compared to random search, and compared performance over time.
    We performed 200 trials on tabular benchmarks.
    \item \textbf{Reporting important details.}
    Local search only has two boolean hyperparameters, so we did not need to tune
    hyperparameters. We reported the times for the full NAS method and all details for
    our experimental setup.
\end{itemize}

%\section{Details from Section~\ref{sec:method}} \label{app:method}
\section{Details from Section~5} \label{app:method}

In this section, we give details from Section~\ref{sec:method}.
For convenience, we restate all theorems and lemmas here.

We start by formally defining all measurable spaces in our theoretical framework.
Recall that the topology of the search space is fixed and discrete, 
while the distribution of validation losses for architectures is randomized and continuous.
This is because training a neural network is not deterministic; in fact, both NASBench-101 and
NASBench-201 include validation and test accuracies for three different random seeds for each
architecture, to better simulate real NAS experiments.
Therefore, we assume that the validation loss for a trained architecture is sampled from
a global probability distribution, and for each architecture, 
the validation losses of its neighbors are sampled from a local probability distribution.

Let $(\R, \B(\R))$ denote a measurable space for the global validation losses induced by the dataset on the architectures,
where $\B(\R)$ is the Borel $\sigma$-algebra on $\R$. 
The distribution for the validation loss of any architecture in the search space is given by $\pdfn(x) \forall x \in \R$. 

Let $(\R^2, \B(\R^2))$ denote a measurable space for validation losses in a neighborhood of an architecture. Let $E: \R^2 \rightarrow \R$ denote a random variable mapping the validation losses of two neighboring architectures to the loss of the second architecture, $E(x,y) \mapsto y$. $E$ has a distribution that is characterized by probability density function $\pdfe(x, y) \forall x, y \in \R$. This gives us a probability over the validation loss for a neighboring architecture.

Every architecture $v \in A$ has a loss $\ell(v) \in \R$ that is sampled from $\pdfn$. For any two neighbors $(v, u) \in E_N$, the PDF for the validation loss $x$ of architecture $u$ is given by $\pdfe(\ell(v), x)$.
Note that choices for the distribution $\pdfe$ are constrained by the 
fixed topology of the search space, as well as the selected distribution $\pdfn$.
Let $(A, 2^A)$ denote a measurable space over the nodes of the graph.

For the rest of this section, we fix an arbitrary neighborhood graph
$G_N$ with vertex set $A$ such that for all $v\in A$, $|N(v)|=s$, i.e.,
$G_N$ has regular degree $s$, and we assume that $G_N$ is vertex transitive.
Each vertex in $A$ is assigned a validation loss according to
$\pdfn$ and $\pdfe$ defined above.
The expectations in the following theorem and lemmas are over the random
draws from $\pdfn$ and $\pdfe$.

\probopt*

\begin{proof}[\textbf{Proof.}]
To prove the first statement,
we introduce an indicator random variable on the architecture space to test if the architecture is a local minimum $I: A \rightarrow \R$,
where 

\begin{align*}
I(v) &= \ind\{\ls^*(v) = v\} \\
&= \ind\{\ell(v) < \ell(u)~\forall u\text{ s.t. }(u,v)\in E_N\}.
\end{align*}

The expected number of local minima in $|A|$ is equal to $|A|$ times the
fraction of nodes in $A$ which are local minima. Therefore, we have

\begin{align*}
&\E[|\{v \in A\mid \ls^*(v) = v\}|]\\
&= n \cdot\P(\{I = 1\}) \\
&= n\int_{-\infty}^{\infty} \pdfn(x)\\
&\quad\cdot \P(\{x<\ell(u) \forall u\text{ s.t. }(u,v)\in E_N, x=\ell(v)\}) dx \\
&= n\int_{-\infty}^{\infty} \pdfn(x) \left(\int_{x}^{\infty} \pdfe(x,y) dy\right)^s dx
\end{align*}

In line one we use the notation $\P(\{I = 1\}) \equiv \P(\{v \in A \mid I(v) = 1\})$.

To prove the second statement, we introduce an indicator random variable on the architecture space that tests if a node will terminate on a local minimum that is
within $\epsilon$ of the global minimum, $I_\epsilon: A \rightarrow \R$, where

\begin{align*}
    I_\epsilon(v) &= \ind\{\ls^*(v) = u \land l(u) - l(v^*) \leq \epsilon\} \\
    &= \ind\{\exists S \in \{\ls^{-*}(u): \ls^*(u) = u \\
    &\quad\quad\land l(u) - l(v^*) \leq \epsilon\}, v \in S\}
\end{align*}

We use this random variable to prove the second statement of the theorem.

\begin{align*}
&\E[|\{v \in A \mid \ell(\ls^*(v))-\ell(v^*)\leq\epsilon\}|]\\
&= n \cdot \P(\{I_\epsilon = 1\})\\
&= n \int_{\ell(v^*)}^{\ell(v^*)+\epsilon}
\P(\{v\in A \mid I(v) = 1, \ell(v) = x\})\\
&\quad\cdot\E[|\ls^{-*}(x)|]dx\\
&= n\int_{\ell(v^*)}^{\ell(v^*)+\epsilon}\pdfn(x)
\left(\int_{\ell(v)}^\infty \pdfe(x, y) dy\right)^s\\
&\quad\cdot\E[|\ls^{-*}(x)|]dx\\
\end{align*}

where the last equality follows from the first half of this theorem.
This concludes the proof.
\end{proof}

Recall that we defined the \emph{branching fraction} of graph $G_N$ as $b_k=|N_k(v)|/\left(|N_{k-1}(v)|\cdot|N(v)|\right)$,
where $N_k(v)$ denotes the set of nodes which are distance $k$ to $v$ in $G_N$.
For example, the branching fraction of a tree with degree $d$ is $1$ for all $k$,
and the branching fraction of a clique is $b_1=1$ and $b_k=0$ for all $k>1.$
Also, for any graph, $b_1=1$.
The neighborhood graph of 
the NASBench-201 search space is $(K_5)^6$ and therefore its branching factor 
is $b_k=\frac{6-k+1}{6k}.$

Now we use Theorem~\ref{thm:prob_opt}
along with Chebyshev's Inequality~\citep{chebyshev1867valeurs}
to show that, in the case when the validation error of each architecture 
has Gaussian noise $\epsilon\sim\mathcal{N}(0,\sigma^2)$, then we can bound the 
expected number of local minima by $O\left(\sigma^{2s}\right).$

\begin{corollary}
Given $|A|=n$, $\ell,s,\text{pdf}_n,$ and $\text{pdf}_e$,
assume that for each architecture $x$, the accuracy is
$\ell(x)+\epsilon$, where $\epsilon\sim\mathcal{N}(0,\sigma^2)$.
Then $\E(\text{\# local minima})\in O\left(\sigma^{2s}\right).$
\end{corollary}

\begin{proof}

Recall Chebyshev's inequality: for a Gaussian random variable
$\epsilon\sim\mathcal{N}(\mu,\sigma^2)$, we have 
$P(|\epsilon| \geq L) \leq \frac{\sigma^2}{L^2}.$
If $\epsilon\sim\mathcal{N}(0,\sigma^2)$, then $\epsilon$ is a symmetric
random variable, so we have
$P(\epsilon \geq L) \leq \frac{\sigma^2}{2L^2}.$
We prove the statement using a modification to the first part of the proof
of Theorem~\ref{thm:prob_opt}.

\begin{align*}
&\E[|\{v \in A\mid \ls^*(v) = v\}|]\\
&= n \cdot \int_{-\infty}^\infty pdf_n(x)\\
&\quad\cdot\P(x < \ell(u)+\epsilon~\forall u\text{ s.t. }(u,v)\in E_N,x=\ell(v))dx\\
&=n \cdot \int_{-\infty}^\infty pdf_n(x) \int_{-\infty}^\infty pdf_e(x,y)\\
&\quad\cdot\left(\P(\epsilon\geq x-y)\right)^s dy dx\\
&\leq n \cdot \int_{-\infty}^\infty pdf_n(x) \int_{-\infty}^\infty pdf_e(x,y)\left(\frac{\sigma^2}{2(x-y)^2}\right)^s dy dx\\
&=\sigma^{2s}\cdot n \cdot \int_{-\infty}^\infty pdf_n(x) \\
&\quad\cdot\int_{-\infty}^\infty pdf_e(x,y)\left(\frac{1}{2(x-y)^2}\right)^s dy dx\\
\end{align*}
\end{proof}

Next, we restate and prove Lemma~\ref{lem:gen_eqns}, which
gives a formula for the $k$'th preimage of the local search function.

\geneqns*

\begin{proof}[\textbf{Proof.}]
The function $\ls^{-1}(v) \in 2^A$ returns a set of nodes which form the 
preimage of node $v \in A$, namely, the set of all neighbors $u \in N(v)$ 
with higher validation loss than $v$, and whose neighbors $w \in N(u)$ excluding $v$ 
have higher validation loss than $\ell(v)$. Formally,
\begin{align*}
&\ls^{-1}(v)\\
&= \{u \in A \mid \ls(u) = v\} \\
&= \{u \in A \mid (v, u) \in E_N, \ell(v) < \ell(u),\\
&\quad~~\{v' \in A \backslash \{v\} \mid (v', u) \in E_N, \ell(v') < \ell(v)\} = \emptyset\}.
\end{align*}

Let $\ls_v^{-1}: A \rightarrow \R$ denote a random variable where $\ls_v^{-1}(u) = \ind\{u \in \ls^{-1}(v)\}$. The probability distribution for $\ls_v^{-1}$ gives 
the probability that a neighbor of $v$ is in the preimage of $v$. 
We can multiply this probability by $|N(v)| = s$ to express the expected number of nodes in the preimage of $v$.

\begin{align*}
&\E[|\ls^{-1}(v)|]\\
&= s \cdot\P(\{\ls_v^{-1} = 1\}) \\
&= s \int_{\ell(v)}^{\infty} \pdfe(\ell(v),y) \left(\int_{l(v)}^{\infty} \pdfe(y,z) dz\right)^{s-1} dy.
\end{align*}

Note that the inner integral is raised to the power of $s-1$, not $s$,
so as not to double count node $v$.
We can use this result to find the preimage of node $v$ after $m$ steps. Let $\ls_v^{-m}: A \rightarrow \R$ denote a random variable where

\begin{align*}
\ls_v^{-m}(u) &= \ind\{u \in \ls^{-m}(v)\} \\
&= \ind\{\forall w \in \ls^{-1}(v), u \in \ls^{-(m-1)}(w)\}. \\
\end{align*}

Following a similar argument as above, we 
compute the expected size of the m'th preimage set.

\begin{align*}
&\E[|\ls^{-m}(v)|]\\
&= b_{k-1}\cdot\E[|\ls^{-1}(v)|]\\
&\quad\cdot \E[|\{\forall w \in A \mid \forall u \in \ls^{-1}(v), 
\ls_u^{-(m-1)}(w) = 1\}|] \\
&= b_{k-1}\cdot\E[|\ls^{-1}(v)|]\\
&\quad\cdot\left( \frac{\int_{\ell(v)}^\infty \pdfe(\ell(v), y) \E[|\ls^{-(m-1)}(y)|] dy}{\int_{\ell(v)}^{\infty}\pdfe(\ell(v), y) dy}\right)
\end{align*}
\end{proof}

%%%%%%%%%%%%%%%%%%%%%%%%%%%%%%%%%%%%%%%%%%%%%%%%%%%%%%%%%%%%%%%%%%%%%%%%%%%%%%%%
%%%%%%%%%%%%%%%%%%%%%%%%%%%%%%%%%%%%%%%%%%%%%%%%%%%%%%%%%%%%%%%%%%%%%%%%%%%%%%%%%
\paragraph{Closed-form solution for single-variate PDFs.}
Now we give the details for Lemma~\ref{lem:full_uniform}.
We start with a lemma that will help us prove Lemma~\ref{lem:full_uniform}.
This lemma uses induction to derive a closed-form solution to
Lemma~\ref{lem:gen_eqns} in the case where $\pdfe(x,y)$ is independent of $x$.

\begin{restatable}{relem}{exindependent}\label{lem:exindependent}
Assume there exists a function $g$ such that $\pdfe(x,y)=g(y)$ for all $x$.
Given $v\in A$, for $k\geq 1$, 
\begin{equation*}
    \E[|\ls^{-k}(v)|]=s^k\left(\int_{\ell(v)}^\infty g(y)dy\right)^{sk} \cdot\prod_{i=0}^{k-1}\frac{b_i}{is+1}.
\end{equation*}
\end{restatable}

\begin{proof}[\textbf{Proof.}]
Given $v\in A$, 
\begin{align*}
&\E[|\ls^{-1}(v)|]\\
&=s\int_{\ell(v)}^\infty \pdfe(\ell(v),y)
\left(\int_{\ell(v)}^\infty \pdfe(y,z)dz\right)^{s-1}dy\\
&=s\int_{\ell(v)}^\infty g(y)\left(\int_{\ell(v)}^\infty g(z)dz\right)^{s-1}dy\\
&=s\left(\int_{\ell(v)}^\infty g(y)dy\right)^s,
\end{align*}
where the first equality follows from Lemma~\ref{lem:gen_eqns}.
Now we give a proof by induction for the closed-form equation.
The base case, $m=1$, is proven above.
Given an integer $m\geq 1$, assume that
\begin{equation*}
\E[|\ls^{-m}(v)|]=s^m\left(\int_{\ell(v)}^\infty g(y)dy\right)^{sn}
\cdot\prod_{i=0}^{m-1}\frac{b_i}{is+1}.
\end{equation*}
Then
\begin{align*}
&\E[|\ls^{-(m+1)}(v)|]\\
&=b_n\cdot\E[|\ls^{-1}(v)|]\cdot 
\left(\int_{\ell(v)}^\infty g(y)dy\right)^{-1}\\
&\quad\cdot\int_{\ell(v)}^\infty g(y)\E[|\ls^{-m}(y)|]dy\\
&=b_n\cdot s\left(\int_{\ell(v)}^\infty g(y)dy\right)^s
\left(\int_{\ell(v)}^\infty g(y)dy\right)^{-1}\\
&\quad\cdot\int_{\ell(v)}^\infty g(y)\cdot \E[|\ls^{-m}(y)|]dy\\
&=b_n\cdot s\left(\int_{\ell(v)}^\infty g(y)dy\right)^{s-1}
\int_{\ell(v)}^\infty g(y)\\
&\quad\cdot s^m\left(\int_{y}^\infty g(z)dz\right)^{sn}
\cdot\prod_{i=0}^{m-1}\frac{b_i}{is+1}\cdot dy\\
&=b_n\cdot s^{m+1}\left(\int_{\ell(v)}^\infty g(y)dy\right)^{s-1}\\
&\quad\cdot\prod_{i=0}^{m-1}\frac{b_i}{is+1}
\int_{\ell(v)}^\infty g(y)\left(\int_{y}^\infty g(z)dz\right)^{sn} dy
\end{align*}
\begin{align*}
&=b_n\cdot s^{m+1}\left(\int_{\ell(v)}^\infty g(y)dy\right)^{s-1}\\
&\quad\cdot\prod_{i=0}^{m-1}\frac{b_i}{is+1}
\left(\int_{\ell(v)}^\infty g(z)dz\right)^{sn+1}\frac{1}{sn+1}\\
&=s^{m+1}\left(\int_{\ell(v)}^\infty g(y)dy\right)^{s(m+1)}
\cdot\prod_{i=0}^{m}\frac{b_i}{is+1}.
\end{align*}

In the first equality, we used Lemma~\ref{lem:gen_eqns}, and
in the fourth equality, we used the fact that
\begin{align*}
&\frac{\partial}{\partial y}\left(\int_{y}^\infty g(z)dz\right)^{sm+1}\\
&\quad=g(y)\left(\int_{y}^\infty g(z)dz\right)^{sm}(sm+1).
\end{align*}

This concludes the proof.
\end{proof}

Next, we prove a lemma which gives strong approximation guarantees
on the size of the full preimage of an architecture, again assuming
that $\pdfe(x,y)$ is independent if $x$.
For this lemma , we need to assume that $n$ is large compared to $s$.
However, this is the only lemma that assumes $n$ is large. In particular,
Lemma~\ref{lem:full_uniform} will not need this assumption.

\begin{restatable}{relem}{preimages}\label{lem:preimages}
Assume there exists $g$ such that $\pdfe(x,y)=g(y)$ for all $x$.
Denote $G(x)=\int_{x}^\infty g(y)dy$.
Given $s$, there exists $N$ such that for all $n>N$,
for all $v$, we have
\begin{align*}
&1+s\cdot G(\ell(v))^s e^{\frac{s}{s+1} G(\ell(v))^s}\\
&\leq \E[|\ls^{-*}(v)|]
\leq 1+s\cdot G(\ell(v))^s\cdot e^{G(\ell(v))^s}.
\end{align*}
\end{restatable}

\begin{proof}[\textbf{Proof.}]
From Lemma~\ref{lem:exindependent},
we have 
\begin{align}
&\E[|\ls^{-*}(v)|]\nonumber\\
&=\sum_{m=1}^\infty \E[|\ls^{-m}(v)|]\nonumber\\
&=\sum_{m=1}^\infty \left(s^m G(\ell(v))^{sm} 
\cdot\prod_{i=0}^{m-1}\frac{b_i}{is+1}\right).\label{eq:full_preimage}
\end{align}

We start with the upper bound.
For all $j\geq 1$, $\frac{s\cdot b_j}{js+1}
\leq \frac{s}{js+1}\leq\frac{s}{sj}=\frac{1}{j}$
because $0\leq b_j\leq 1$ for all $1\leq j$.
Therefore for all $i$,
\begin{equation*}
\prod_{j=1}^{i-1}\frac{s\cdot b_j}{js+1}\leq\prod_{j=1}^{i-1}\frac{1}{j}
=\frac{1}{(i-1)!}.
\end{equation*}

It follows that
\begin{align*}
\E[|\ls^{-m}(v)|]&=\sum_{i=1}^\infty s^i G(\ell(v))^{si} \prod_{j=0}^{i-1}\frac{b_j}{js+1}\\
&=s\sum_{i=1}^\infty G(\ell(v))^s\prod_{j=1}^{i-1}\frac{s\cdot b_j}{js+1}\\
&\leq s G(\ell(v))^s\sum_{i=1}^\infty\left(\frac{1}{(i-1)!}\cdot G(\ell(v))^s\right)\\
&=s G(\ell(v))^s e^{G(\ell(v))^s}.
\end{align*}
The final equality comes from the well-known Taylor
series $e^x=\sum_{n=0}^\infty \frac{x^n}{n!}$ 
(e.g.\ \citep{abramowitz1948handbook}) evaluated at $x=G(\ell(v))^s$.

Now we prove the lower bound.
$b_1=1$ by definition for all graphs, and for $1<j\leq D$,
$0\leq b_j\leq 1$, where $D$ denotes the diameter of the graph.
(Since $N_D(v)=n$ for all $v$, $b_j$ is meaningless for $j\geq D$.)
Recall that all of our arguments assume vertex transitivity.
It follows that $b_{D-1}\leq b_{D-2}\leq\cdots\leq b_1$.
Now, for a fixed $s$, $b_{D-1}$ approaches 1 as $n$ approaches infinity. 
Therefore, given $s$, there exists $N$ such that for all $n>N$,
$b_{D-1}>\frac{2s+1}{2(s+1)}.$
Then for all $i$,
\begin{align*}
\frac{1}{j(s+1)}&\leq\frac{1}{js+1}\left(\frac{js+1}{j(s+1)}\right)
\leq\frac{1}{js+1}\left(\frac{2s+1}{2(s+1)}\right)\\
&\leq\frac{1}{js+1}\left(b_{D-1}\right)\leq\frac{b_j}{js+1}.
\end{align*}

Therefore,
\begin{equation*}
\frac{s^{i-1}}{(i-1)!(s+1)^{i-1}}
=\prod_{j=1}^{i-1}\frac{s}{j(s+1)}
\leq \prod_{j=1}^{i-1}\frac{s\cdot b_j}{js+1}.
\end{equation*}

It follows that
\begin{align*}
&\E[|\ls^{-*}(v)|]\\
&=\sum_{i=1}^\infty s^i G(\ell(v))^{si} \prod_{j=0}^{i-1}\frac{b_j}{js+1}\\
&=s\sum_{i=1}^\infty G(\ell(v))^{si}\prod_{j=1}^{i-1}\frac{s\cdot b_j}{js+1}\\
&\geq s G(\ell(v))^s\sum_{i=1}^\infty 
\left(\frac{1}{(i-1)!}\cdot\left(\frac{s G(\ell(v))^s}{s+1}\right)^{i-1}\right)\\
&=s G(\ell(v))^s\cdot e^{\frac{s}{s+1}G(\ell(v))^s}.
\end{align*}
The final equality again comes from the Taylor series $e^x$, 
this time evaluated at $x=\frac{s}{s+1}\cdot G(\ell(v))^s.$
\end{proof}

Note that Equation~\ref{eq:full_preimage} does not require the assumption
that $n$ is large. Now we can use Equation~\ref{eq:full_preimage} and
Theorem~\ref{thm:prob_opt} to prove Lemma~\ref{lem:full_uniform}.

\fulluniform*

\begin{proof}[\textbf{Proof.}]

The probability density function of $U([0,1])$ is equal to 1 on $[0,1]$ and 0 otherwise.
Let $\ell(v)=x$.
Then $\int_{x}^\infty \pdfe(x,y)dy=\int_{x}^1 dy=(1-x).$
Using Theorem~\ref{thm:prob_opt},
we have 
\begin{equation*}
\E[|\{v\mid v=\ls^*(v)\}|]
=n\int_{\ell(v^*)}^\infty 1\cdot 
\left(1-x\right)^s dx=\frac{n}{s+1}.
\end{equation*}

Now we plug in Equation~\ref{eq:full_preimage}
%Lemma~\ref{lem:exindependent} 
to the second part of Theorem~\ref{thm:prob_opt}.
\begin{align*}
&\E[|\{v\mid \ell(\ls^*(v))-\ell(v^*)\leq\epsilon\}|]\\
&=n \int_{\ell(v^*)}^{\ell(v^*)+\epsilon}
1\cdot (1-x)^s \sum_{k=0}^\infty \E[|\ls^{-k}(x)|]dx\\
&=n \int_{\ell(v^*)}^{\ell(v^*)+\epsilon}
(1-x)^s \sum_{k=0}^\infty \left(s^k(1-x)^{sk} 
%\cdot
\prod_{i=0}^{k-1}\frac{b_i}{is+1}\right)dx\\
&=n \sum_{k=0}^\infty\left(\frac{s^k\left(1-(1-\epsilon)^{(k+1)s+1}\right)}{(k+1)s+1}
\cdot\prod_{i=0}^{k-1}\frac{b_i}{is+1}\right).
\end{align*}
This concludes the proof.
\end{proof} 

%\end{comment}

\end{document}